%% file: main.tex
\documentclass[wcp]{jmlr}

\input{packages}

\input{notations}

\newif\ifsup
\supfalse\suptrue		

\usepackage[
backgroundcolor = White,textwidth=\marginparwidth]{todonotes}

\usepackage{longtable}

\usepackage{booktabs}

\jmlrvolume{129}
\jmlryear{2020}
\jmlrworkshop{ACML 2020}

\title[Thompson Sampling for Unsupervised Sequential Selection]{Thompson Sampling for Unsupervised Sequential Selection}

\author{\Name{Arun Verma} 
		\Email{v.arun@iitb.ac.in}\\
	\Name{Manjesh K. Hanawal} 
		\Email{mhanawal@iitb.ac.in}\\
	\Name{Nandyala Hemachandra} 
		\Email{nh@iitb.ac.in}\\
	\addr Indian Institure of Technology Bombay, India
}

\editors{Sinno Jialin Pan and Masashi Sugiyama}
\ifsup
\fi

\begin{document}

	\maketitle

	\begin{abstract}
		 Thompson Sampling has generated significant interest due to its better empirical performance than upper confidence bound based algorithms. In this paper, we study Thompson Sampling based algorithm for {\em \underline{U}nsupervised \underline{S}equential \underline{S}election} (USS) problem. The USS problem is a variant of the stochastic multi-armed bandits problem, where the loss of an arm can not be inferred from the observed feedback. In the USS setup, arms are associated with fixed costs and are ordered, forming a cascade. In each round, the learner selects an arm and observes the feedback from arms up to the selected arm. The learner's goal is to find the arm that minimizes the expected total loss. The total loss is the sum of the cost incurred for selecting the arm and the stochastic loss associated with the selected arm. The problem is challenging because, without knowing the mean loss, one cannot compute the total loss for the selected arm. Clearly, learning is feasible only if the optimal arm can be inferred from the problem structure. As shown in the prior work, learning is possible when the problem instance satisfies the so-called `Weak Dominance' $(\WD)$ property. Under $\WD$, we show that our Thompson Sampling based algorithm for the USS problem achieves near-optimal regret and has better numerical performance than existing algorithms. 
	\end{abstract}


	\section{Introduction}
	\label{sec:introduction}
	\input{introduction}

	\section{Problem Setting}		
	\label{sec:problemSetting}
	\input{problem_setting}

	\section{Conditions for Learning Optimal Arm}
	\label{sec:learning}
	\input{learning}

	\section{Thompson Sampling based Algorithm for USS}
	\label{sec:ts_uss}
	\input{uss_ts}

	\section{Experiments}
	\label{sec:experiments}
	\input{experiment}

	\section{Conclusion}
    We studied the unsupervised sequential selection (USS) problem, where both accuracy and cost of using arms are important. It is a variant of the stochastic partial monitoring problem, where the losses are not observed. Still, one can compare the feedback of two arms to see if they agree or disagree. We estimate the disagreement probability between each pair of the arms and develop an algorithm named \ref{alg:TS_USS} that achieves near-optimal regret. We demonstrate our algorithms' performance on two real datasets and empirically show that any problem instance satisfying WD property has sub-linear regret. 
	We ignored the inherent side observations due to the arms' cascade structure. By using these side observations, one can tighten the regret bounds. Another interesting future direction is to develop algorithms that relax the cascade structure assumption and selects the best subset of arms.

	 \section*{Acknowledgments}
	 Manjesh K. Hanawal would like to thank the support from INSPIRE faculty fellowship from DST and Early Career Research (ECR) Award from SERB, Govt. of India.

	\bibliography{ref}

	\ifsup
		\newpage
		\centerline{
			\Large \bf Supplementary Material for }~\\
		\centerline{
			\Large \bf `Thompson Sampling for Unsupervised Sequential Selection'}
		\hrulefill \\
		
		\appendix
		\label{sec:appendix}
		\input{appendix}
	\fi

\end{document}

%% file: packages.tex
\usepackage[utf8]{inputenc} 
\usepackage[T1]{fontenc}    
\usepackage{hyperref}       
\usepackage{url}            
\usepackage{booktabs}       
\usepackage{amsfonts}       
\usepackage{nicefrac}       
\usepackage{microtype}      

\usepackage{graphicx}
\usepackage{rotating}
\usepackage{pdflscape}

\usepackage{amsmath,amsthm,amssymb}
\usepackage{algorithm,algorithmic}
\usepackage{bbm,dsfont}
\usepackage{cancel}
\usepackage{enumerate, cases}
\usepackage{thmtools,thm-restate}
\usepackage[none]{hyphenat}
\usepackage{natbib}
\usepackage{mathtools}
\usepackage{array}
\usepackage{multicol,multirow}
\newcolumntype{M}[1]{>{\centering\arraybackslash}m{#1}}
\usepackage{pifont}

\allowdisplaybreaks

\hypersetup{
	colorlinks	= true,		
	urlcolor     = blue,	 
	linkcolor	 = purple, 	 
	citecolor    = violet    	 
}
\usepackage[capitalize]{cleveref}

\allowdisplaybreaks


%% file: notations.tex
\newcommand{\ist}{{i^\star}}
\newcommand{\USS}{\cP_{\small{\text{USS}}}}
\newcommand{\bc}{\boldsymbol{c}}
\newcommand{\bd}{\boldsymbol{Q}}

\newcommand{\SD}{\mathrm{SD}}
\newcommand{\WD}{\mathrm{WD}}
\newcommand{\PSD}{\cP_{\SD}}
\newcommand{\PWD}{\cP_{\WD}}

\newcommand{\Yti}{Y_t^i}
\newcommand{\Ytj}{Y_t^j}

\newcommand{\Yt}{Y_t}
\newcommand{\Yi}{Y^i}
\newcommand{\Yj}{Y^j}
\newcommand{\Yis}{Y^{i^\star}}

\newcommand{\Ht}{\mathcal{H}_t}

\newcommand{\pij}{p_{ij}}
\newcommand{\pijt}{p_{ij}^{(t)}}
\newcommand{\pjis}{p_{ji^\star}}
\newcommand{\pijs}{p_{i^\star j}}
\newcommand{\pijts}{p_{i^\star j}}
\newcommand{\opij}{\tilde{p}_{ij}^{(t)}}

\newcommand{\opji}{\tilde{p}_{ji}^{(t)}}

\newcommand{\opijtst}{\tilde{p}_{ij}^{(t)}}
\newcommand{\opijsts}{\tilde{p}_{i^\star j}^{(t)}}

\newcommand{\hpijs}{\hat{p}_{i^\star j}}
\newcommand{\hpijst}{\hat{p}_{i^\star j}^{(t)}}

\newcommand\numberthis{\addtocounter{equation}{1}\tag{\theequation}}

\newcommand{\Regret}{\mathfrak{R}}

\newcommand{\EE}[1]{\bE\left[#1\right]}

\newcommand{\Prob}[1]{\bP\left\{#1\right\}}

\newcommand{\one}[1]{\mathds{1}_{\left\{#1\right\}}}

\renewcommand{\phi}{\varphi}
\renewcommand{\epsilon}{\varepsilon}

\newcommand{\al}[1]{ \begin{align} #1  \end{align}}
\newcommand{\eq}[1]{ \begin{equation} #1  \end{equation}}

\newcommand{\eqs}[1]{ \begin{equation*} #1  \end{equation*}}

\newcommand{\el}{\end{flushleft}}
\newcommand{\bl}{\begin{flushleft}}

\newcommand{\argmin}{\arg\!\min}

\newcommand{\bE}{\mathbb{E}}

\newcommand{\bP}{\mathbb{P}}

\newcommand{\cB}{\mathcal{B}}

\newcommand{\cP}{\mathcal{P}}

\theoremstyle{plain}

\newtheorem{lem}{Lemma}
\newtheorem{prop}{Proposition}
\newtheorem{cor}{Corollary}

\newtheorem{rem}{Remark}
\newtheorem{defi}{Definition}

\newtheorem{fact}{Fact}

%% file: introduction.tex

Many variants of sequential decision-making problems are considered in the literature depending on the type of feedback and the amount of information they reveal about the rewards. The multi-armed bandits and the expert setting \citep{ML02_auer2002finite, NOW12_bubeck2012regret} are well-studied problems where feedback provides direct information about the rewards. In the multi-armed bandit setting, feedback observed from an action reveals only the reward associated with that action. However, in the expert setting, the feedback observed from an action reveals reward associated with the action played as well as all other actions. The settings that span in between these two extreme cases are also studied, namely, bandits with side-information \citep{NIPS11_mannor2011bandits,NIPS13_alon2013bandits,COLT15_alon2015online,NIPS15_wu2015online}. In many problems, the actions can be indirectly tied to the rewards. Such setting is referred as partial monitoring setting \citep{MOR06_cesa2006regret,ALT12_bartok2012partial,MOR14_bartok2014partial}. It includes all the previously described setups as special cases.

Most of the previous work on partial monitoring is restricted to cases where feedback from the actions allows the learner to identify the rewards of the actions. However, in many areas like crowd-sourcing \citep{NIPS17_bonald2017minimax, ICLM18_kleindessner2018crowdsourcing}, medical diagnosis \citep{AISTATS17_hanawal2017unsupervised}, resource allocation \citep{NeurIPS19_verma2019censored}, and many others, feedback from actions may not even be sufficient to identify their rewards.

Such reward structures can be found in many prediction problems, where one may have to predict labels for instances whose associated ground-truth cannot be obtained. Such problems arise naturally in medical diagnosis, crowd-sourcing, security system \citep{AISTATS17_hanawal2017unsupervised}, and unsupervised features selection \citep{COMSNETS20_verma2020unsupervised}. In the medical diagnosis problem, the true state of the patients may not be known; hence, the test's effectiveness cannot be known. Whereas in the crowd-sourcing systems, the expertise level of self-listed-agents (workers) is unknown; therefore, the quality of their work cannot be known. In these prediction problems, we can observe prediction from test/worker, but we cannot ascertain their reliability due to the absence of ground truth.

In many of the real-world situations like those found in medical diagnosis, airport security, and manufacturing, a set of tests or classifiers is used to monitor patients, people, and products. Tests have cost with the more informative ones resulting in higher monetary costs and higher latency. Thus, they are often organized as a cascade \citep{AISTATS12_chen2012classifier, AISTATS13_trapeznikov2013supervised}, so that a new input is first probed by an inexpensive test then more expensive one. We refer to such cascaded systems as {\it Unsupervised Sequential Selection} (USS) problem\footnote{Note that the unsupervised sequential selection problem is referred to as the unsupervised sensor selection problem in the prior work \citep{AISTATS17_hanawal2017unsupervised, AISTATS19_verma2019online}.}, where an arm represents a test/ worker. A learner's goal in the USS problem is to select the most cost-effective arm so that the overall system maintains high accuracy at low average costs.

In this paper, we draw upon several concepts introduced in prior work \citep{AISTATS17_hanawal2017unsupervised, AISTATS19_verma2019online}. Specifically, we use the notion of weak dominance \citep{AISTATS19_verma2019online} that helps to find optimal arm using observed disagreements between arms. We propose a Thompson Sampling \citep{COLT12_agrawal2012analysis, ALT12_kaufmann2012thompson, AISTATS13_agrawal2013further} based algorithm for the USS problem and show that it is a near-optimal algorithm.  We then validate its performance on several problem instances derived from synthetic and real datasets.  
Our contributions can be summarized as follows:
\begin{itemize}
	\item We develop a Thompson Sampling based algorithm named \ref{alg:TS_USS} for the USS problem. This algorithm uses a one-sided test to find the optimal arm, whereas the state-of-the-art algorithm proposed in \cite{AISTATS19_verma2019online} uses a two-sided test to identify the optimal arm. The new one-sided test leads to a simpler algorithm. 
	
	\item In \cref{sec:ts_uss}, we characterize the regret of \ref{alg:TS_USS} in terms of how well the problem instance satisfies the $\WD$ property and show that it has sub-linear regret under $\WD$ property. We also give problem independent regret bound and establish that the regret bounds are near-optimal using results from the partial monitoring system.
	
	\item We demonstrate empirical performance of \ref{alg:TS_USS} on synthetic and real datasets in \cref{sec:experiments}. Our experimental results show that regret of \ref{alg:TS_USS} is always lower than USS-UCB \citep{AISTATS19_verma2019online} and heuristic algorithm given in \cite{AISTATS17_hanawal2017unsupervised}.
\end{itemize}

%% file: problem_setting.tex

We consider a stochastic $K$-armed bandits problem. The set of arms is denoted by $[K]$ where $[K] \doteq \{1,2,\ldots, K\}$. In each round $t$, the environment generates a binary $K+1$-dimensional vector $\left(\Yt, \{\Yti\}_{i \in [K]}\right)$. The variable $\Yt$ denotes the best binary feedback for round $t$, which is hidden from the learner. The vector $\left(\{\Yti\}_{i \in [K]}\right) \in \{0, 1\}^{K}$ represents observed feedback at time $t$, where $\Yti$ denote the feedback\footnote{In the USS setup, an arm $i$ could represent a classifier. After using the first $i$ classifiers, the final label can be a function of labels predicted by the first $i$ classifiers, $i \in[K]$.} observed after playing arm $i$. We denote the cost for using arm $i \in [K]$ as $c_i\geq 0$ that is known to learner and the same for all rounds.

In the USS setup, the arms are assumed to be ordered and form a cascade. When the learner selects an arm $i \in [K]$, the feedback from all arms till arm $i$ in the cascade is observed. The expected loss of playing the arm $i$ is denoted as $\gamma_i \doteq \EE{\one{\Yi \neq Y}} = \Prob{\Yi \neq Y}$, where $\one{A}$ denotes indicator of event $A$. The {\it expected total cost} incurred by playing arm $i$ is defined as $\gamma_i +\lambda_iC_i$, where $C_i \doteq c_1 + \ldots + c_i$ and $\lambda_i$ is a trade-off parameter that normalizes the loss and the incurred cost of playing arm $i$.

Since the best binary feedback are hidden from the learner, the expected loss of an arm cannot be inferred from the observed feedback. We thus have a version of the stochastic partial monitoring problem, and we refer to it as \underline{u}nsupervised \underline{s}equential \underline{s}election (USS) problem. Let $\bd$ be the unknown joint distribution of $(Y, Y^1, Y^2 \ldots, Y^K)$. Henceforth we identify an USS instance as $P \doteq (\bd,\bc)$ where $\bc \doteq (c_1, c_2, \ldots, c_K)$ is the known cost vector of arms. We denote the collection of all USS instances as $\USS$. For instance $P \in \USS$, the optimal arm is given by
\eq{
	\label{equ:optimalArm}
	\ist \in \max\left\{\argmin _{i \in [K]} \left( \gamma_i+\lambda_iC_i \right)\right\}
}
where the `max' operator selects the arm with the largest index among the minimizers.
The choice of $\ist$ in \cref{equ:optimalArm} is risk-averse as we prefer the arm with lower error among the good arms. The interaction between the environment and a learner is given in Algorithm \ref{alg:USS}.
\begin{algorithm}[H]
	\caption{Learning with USS instance $(\bd, \bc)$}
	\label{alg:USS}
	For each round $t$: 
	\begin{enumerate}
		\item \textbf{Environment} chooses a vector $(\Yt, \{\Yti\}_{i \in [K]})\sim \bd$.
		\item \textbf{Learner} selects an arm $I_t \in [K]$ to stop in cascade.
		\item \textbf{Feedback and Loss:} The learner observes feedback $(\Yt^1, \Yt^2, \ldots, \Yt^{I_t})$ and incurs a total loss $\one{Y^{I_t} \neq \Yt} + \lambda_{I_t}C_{I_t}$. 
	\end{enumerate}
\end{algorithm}

The learner's goal is to learn a policy that find an arm such that the cumulative expected loss is minimized. Specifically, for $T$ rounds, we measure the performance of a policy that selects an arm $I_t$ in round $t$ in terms of regret given by
\eq{
	\label{eq:cum_regret}
	{\Regret_T} = \sum_{t=1}^T\left( \gamma_{I_t} +  \lambda_{I_t}C_{I_t} - \left(\gamma_\ist +  \lambda_\ist C_{\ist} \right) \right).
}

A good policy should have sub-linear regret, i.e.,
$\lim\limits_{T \rightarrow \infty}{\Regret_T}/T = 0$. The sub-linear regret implies that the learner collects almost as much reward in expectation in the long run as an oracle that knew the optimal arm from the first round. 
We say that a problem instance $P \in \USS$ is learnable if there exists a policy with sub-linear regret.

%% file: learning.tex

Next, we define the strong and weak dominance property of the USS problem instance that makes the learning of the optimal arm possible.
\begin{defi}[Strong  Dominance $(\SD)$ \citep{AISTATS17_hanawal2017unsupervised}] 
	\label{def:CSD} 
	A problem instance is said to satisfy $\SD$ property if
	\eqs{
		Y^i = Y \mbox{ for some } i \in[K] \implies  Y^j = Y, ~~ \forall j > i.
	}
	We represent the set of all instances in $\USS$ that satisfy $\SD$ property by $\PSD$.
\end{defi}
 
The $\SD$ property implies that if the feedback of an arm is same as the true reward, then the feedback of all the arms in the subsequent stages of the cascade is also same as the true reward. \cite{AISTATS17_hanawal2017unsupervised} show that the set of all instances satisfying SD property is learnable by mapping such instances to stochastic multi-armed bandits problem with side information \citep{NIPS15_wu2015online}. A weaker version of the $\SD$ property is defined as follows:
\begin{defi}[Weak Dominance $(\WD)$ \citep{AISTATS19_verma2019online}] 
	\label{def:WD} 
	Let $\ist$ denote the optimal arm. Then an instance $P \in \USS$ is said to satisfy {weak dominance property} if
	\eq{
		\label{equ:WDProp}
		\forall j>\ist: C_j - C_\ist > \Prob{Y^\ist \ne \Yj}.		
	}
	We denote the set of all instances in $\USS$ that satisfy $\WD$ property by $\PWD$.
\end{defi}

The set of problems satisfying the $\WD$ property is maximally learnable, and any relaxation of $\WD$ property makes the problem unlearnable \citep[Theorem 1]{AISTATS19_verma2019online}. In the following equation, we use an alternative characterization of the $\WD$ property, given as
\eq{
    \label{def:Xi}
    \xi \doteq \min_{j>\ist}\left\{C_j - C_\ist - \Prob{Y^\ist \ne \Yj} \right\} > 0.
}
The larger the value of $\xi$, `stronger' is the $\WD$ property, and easier to identify an optimal arm. We later characterize the regret upper bound of our algorithm in terms of $\xi$.

\subsection{Optimal Arm Selection}
Without loss of generality, we set $\lambda_i=1$ for all $i\in [K]$ as their value can be absorbed into the costs. Since $\ist = \max\big\{\arg\min\limits_{i \in [K]}\left(\gamma_i+ C_i \right)\big\}$, it must satisfy following equation:
\begin{subequations}
	\label{eq:cost_exp_err}
	\al{
		&\forall j<\ist \,:\, C_\ist - C_j \leq \gamma_j-\gamma_\ist \,, \label{eq:wd1}\\ 
		&\forall j>\ist \,:\, C_j - C_\ist > \gamma_\ist - \gamma_j \,. \label{eq:wd2}
	}
\end{subequations}

As the loss of an arm is not observed, the above equations can not lead to a sound arm selection criteria. We thus have to relate the unobservable quantities in terms of the quantities that can be observed. In our setup, we can compare the feedback of two arms, which can be used to estimate their disagreement probability. For notation convenience, we define $\pij \doteq \Prob{\Yi \ne \Yj}$. The value of $\pij$ can be estimated as it is observable. We use the following result from \cite{AISTATS17_hanawal2017unsupervised} that relates the differences in the unobserved error rates in terms of their observable disagreement probability.
\begin{prop}[Proposition 3 in \cite{AISTATS17_hanawal2017unsupervised}]
	\label{lem:err_prob_contx}
	For any two arms $i$ and $j$, $\gamma_{i} - \gamma_{j} = \pij - 2\Prob{\Yi = Y, \Yj \ne Y}$.
\end{prop}

\noindent
Now, using \cref{lem:err_prob_contx}, we can replace  \cref{eq:wd1} by
\eq{
	\label{eq:selectDisProbLow}
	\forall j<\ist \,:\, C_{\ist} - C_j \leq  \pjis,
}
which only has observable quantities. For $j>\ist$, we can replace \cref{eq:wd2} by using the $\WD$ property as follows:
\eq{
	\label{eq:selectDisProbHigh}
	\forall j>\ist \,:\, C_j - C_{\ist} >  \pijs.
}

\noindent
Using \cref{eq:selectDisProbLow} and \cref{eq:selectDisProbHigh}, our next result gives the optimal arm for a problem instance.
\begin{restatable}{lem}{SetBx}
	\label{lem:Bx}
	Let $P \in \PWD$ and $\cB = \left\{i: \forall j>i, C_j - C_i > \pij \right\}\cup \{K\}$. Then the arm $I_t =\min(\cB)$ is the optimal arm for the problem instance $P$.
\end{restatable}
\begin{proof}
	Let $\ist$ be an optimal arm for the problem instance $P$. Since $\pijs \doteq \Prob{\Yis \ne\Yj}$, we have $\forall j<\ist:\, C_{\ist} - C_j \le \Prob{\Yis \ne\Yj} \implies C_{\ist} - C_j \ngtr\Prob{\Yis \ne\Yj} \implies j \notin \cB, \forall j < \ist$.  If any sub-optimal arm $h \in \cB$ then the index of arm $h$ must be larger than the index of optimal arm $\ist$ in the cascade. Hence the element of the set $\cB$ in round $t$ is given as follows:
	\eqs{
		\cB = \{\ist, h_1, \ldots, h_t, K\},
	} 	
	where $\ist < h_1 < \cdots < h_t < K$.	By construction of set $\cB$, the minimum indexed arm in set $\cB$ is the optimal arm.
\end{proof}

\begin{rem}
	The $\WD$ property holds trivially for the problem instances that satisfy $\SD$ property as the difference of mean losses is the same as the disagreement probability between two arms due to $\Prob{\Yti = \Yt, \Ytj \ne \Yt} = 0$ for $j>i$. Also, by definition, the $\WD$ property holds for all problem instances where the last arm of the cascade is an optimal arm.
\end{rem}

%% file: uss_ts.tex

Upper Confidence Bound (UCB) based methods are useful for dealing with the trade-off between exploration and exploitation in bandit problems  \citep{ML02_auer2002finite, COLT11_garivier2011kl}. UCB has been widely used for solving various sequential decision-making problems. On the other hand, Thompson Sampling (TS) is an online algorithm based on Bayesian updates. TS selects an arm to play according to its probability of being the best arm, and it is shown that TS is empirically superior then UCB based algorithms for various MAB problems \citep{NIPS11_chapelle2011empirical}. TS also achieves lower bound for MAB when rewards of arms have Bernoulli distribution, as shown by \cite{ALT12_kaufmann2012thompson}.

\subsection{Algorithm: \ref{alg:TS_USS}}
We develop a Thompson Sampling based algorithm, named \ref{alg:TS_USS}, that uses \cref{lem:Bx} to select optimal arm. The algorithm works as follows: It sets the prior distribution of disagreement probability for each pair of arms as the Beta distribution, Beta$(1, 1)$, which is the same as Uniform distribution on $[0,1]$. The variable $S_{ij}$ represents the number of rounds when a disagreement is observed between arm $i$ and $j$. Whereas, the variable $F_{ij}$ represents the number of rounds when an agreement is observed. The variables $S_{ij}^{(t)}$ and $F_{ij}^{(t)}$ denote the values of $S_{ij}$ and $F_{ij}$ at the beginning of round $t$.
\begin{algorithm}[H]
	\renewcommand{\thealgorithm}{USS-TS}
	\floatname{algorithm}{}
	\caption{Thompson Sampling based Algorithm for Unsupervised Sequential Selection} 
	\label{alg:TS_USS}
	\begin{algorithmic}[1]
		\STATE Set $ \forall 1 \le i< j \le K: \mathcal{S}_{ij}^{(1)} \leftarrow 1, \mathcal{F}_{ij}^{(1)} \leftarrow 1$
		\FOR{$t=1,2,...$}
			\STATE Set $i=1$ and $I_t=0$
			\WHILE{$I_t = 0$}
				\STATE Play arm $i$
				\STATE $\forall j \in [i+1, K]:$ compute $\opijtst \leftarrow \mbox{Beta}(\mathcal{S}_{ij}^{(t)}, \mathcal{F}_{ij}^{(t)})$
				\STATE If $\forall j \in [i+1, K]: C_j - C_i > \opijtst$ or $i=K$ then set $I_t = i$ else  set $i=i+1$
			\ENDWHILE
		\STATE Select arm $I_t$ and observe $Y_t^1, Y_t^2, \dots, Y_t^{I_t}$
		\STATE $\forall 1\le i< j \le I_t:$ update $\mathcal{S}_{ij}^{(t+1)} \leftarrow\mathcal{S}_{ij}^{(t)}+ \one{\Yti \ne \Ytj}, \mathcal{F}_{ij}^{(t+1)} \leftarrow\mathcal{F}_{ij}^{(t)} + \one{\Yti = \Ytj}$
		\ENDFOR
	\end{algorithmic}
\end{algorithm}

In round $t$, the learner plays the arm $i=1$ and then observe its feedback. For each $(i,j)$ pair, a sample $\opijtst$ is independently drawn from Beta$(S_{ij}^{(t)}, F_{ij}^{(t)})$. Then algorithm checks whether the arm $i$ is the best arm using \cref{eq:selectDisProbHigh} with $\opijtst$ in place of $\pijt$. If the arm $i$ is not the best, then the algorithm plays the next arm, and the same process is repeated. If the arm $i$ is the best arm for the round $t$, then the algorithm stops at arm $I_t=i$ in the round $t$.

After selecting arm $I_t$, the feedback from arms $1, \ldots, I_t$ are observed, which is used to update the values of $S_{ij}^{(t+1)}$  and $F_{ij}^{(t+1)}$. The same process is repeated in the subsequent rounds.

\begin{rem}
	\ref{alg:TS_USS} is adapted for the USS problem from the Thompson Sampling algorithm for stochastic multi-armed bandits. However, the feedback structure and the way arms are selected in the USS setup differ from that in the stochastic multi-armed bandits.
\end{rem}

\subsection{Analysis}
The following definitions and results are useful in subsequent proof arguments.

\begin{defi}
	For the optimal arm $i^\star$ and $j \in [K]$, define 
	\begin{subnumcases}
	{\xi_j \doteq }
	\pijs - (C_{i^\star}-C_j), \;\text{ if } j<i^\star  \label{def_xi_l}\\
	C_j-C_{i^\star} - \pijs,   \;\;\;\;\text{ if } j>i^\star \label{def_xi_h} 
	\end{subnumcases}
	where $\pijs=\Prob{\Yis = \Yj}$. 
\end{defi}
\noindent
Note that the values of $\xi_j$ for all $j\in [K]$ is positive under the $\WD$ property.

\begin{defi}[Action Preference ($\succ_t$)]
	\label{def:preference}
	\ref{alg:TS_USS} prefers the arm $i$ over arm $j$ in round $t$ if:
	\begin{subnumcases}	
	{i \succ_t j \doteq }
	\opji \ge C_i - C_j &\text{if } j<i \label{def_prefer_l} \\
	\opij < C_j - C_i&\text{if } j>i \label{def_prefer_h}
	\end{subnumcases}
\end{defi}

\begin{defi}[Transitivity Property]
	\label{def:trans_prop}
	If $i \succ_t j$ and $j \succ_t k$ then $i \succ_t k$.
\end{defi}

\begin{defi}
	Let $\Ht$ denote the $\sigma$-algebra generated by the history of selected arms and observations at the beginning of the time $t$ and given as follows:
	\begin{equation*}
		\Ht \doteq \left\{I_s, \left\{ Y_s^i \right\}_{i \le I_s}, s = 1, \ldots, t-1\right\},
	\end{equation*}
	where $I_s$ denotes the arm selected and set $\left\{ Y_s^i \right\}_{i \le I_s}$ denotes the observations from arm $1$ to $I_s$ in the round $s$. Define $\mathcal{H}_1 \doteq \{\}.$
\end{defi}

\begin{fact}[Beta-Binomial equality, Fact 1 in \cite{COLT12_agrawal2012analysis}]
	\label{fact:beta_binomial}
	Let $F_{\alpha,\beta}^{beta}(y)$ be the cumulative distribution function (cdf) of the beta distribution with integer parameters $\alpha$ and $\beta$. Let $F_{n,p}^B(\cdot)$ be the cdf of the binomial distribution with parameters $n$ and $p$. Then,
	\begin{equation*}
		F_{\alpha,\beta}^{beta}(y) = 1 - F_{\alpha+\beta-1,y}^B(\alpha-1).
	\end{equation*}
\end{fact}

\begin{lem}[Lemma 2 in \cite{AISTATS13_agrawal2013further}]
	\label{lem:betaExpBound}
	Let $n \ge 0$ and $\hat{\mu}_n$ be the empirical average of $n$ samples from Bernoulli($\mu$). Let $x< \mu$ and $q_n(x) \doteq 1 -  F_{n\hat{\mu}_n + 1, n(1-\hat{\mu}_n) + 1}^{beta}(x)$ be the probability that the posterior sample from the Beta distribution with its parameter $n\hat{\mu}_n + 1, n(1-\hat{\mu}_n) + 1$ exceeds $x$. Then,
	\begin{equation*}
		\EE{\frac{1}{q_n(x)}-1} \le
		\begin{cases}
			\frac{3}{\Delta(x)} & \mbox{if } n < 8/\Delta(x) \\
			\Theta\left(\exp^{-\frac{n\Delta(x)^2}{2}} + \frac{\exp^{-{nd(x,\mu)}}}{(n+1)\Delta(x)^2} + \frac{1}{\exp^{\frac{n\Delta(x)^2}{4}} - 1} \right) & \mbox{if } n \ge 8/\Delta(x),
		\end{cases}
	\end{equation*}
	where $\Delta(x) \doteq \mu - x$ and $d(x, \mu) \doteq x\log\left(\frac{x}{\mu}\right) + (1-x)\log\left(\frac{1-x}{1-\mu}\right)$.
\end{lem}

Recall that $\pijs$ is the disagreement probability between arm $i^\star$ and $j$ and $\opijsts$ is the sample of $\pijs$ using Beta distribution with the $t$ samples. 
Next, we bound the probability by which \ref{alg:TS_USS} selects the sub-optimal arm whose index is smaller than the optimal arm.
\begin{defi}
	\label{def:q}
	For any $j < i^\star$, define $q_{j,t}$ as the probability
	\begin{equation*}
		q_{j,t} \doteq \Prob{\opijsts \ge \pijs - \xi_j|\Ht}.
	\end{equation*}
\end{defi}

\begin{lem}
	\label{lem:relationLowerOptimal}
	Let $P \in \PWD$ and satisfies the transitivity property. If $j < i^\star$ then the probability by which \ref{alg:TS_USS} selects any sub-optimal arm $j$ over the optimal arm is given by
	\begin{equation*}
		\Prob{I_t =j, j< \ist|\Ht} \le \frac{(1-q_{j,t})}{q_{j,t}}\Prob{I_t \ge i^\star|\Ht}.
	\end{equation*}
\end{lem}

\begin{proof}
	If the sub-optimal arm $j$ is selected then arm $j$ is preferred over the arms whose indexed is larger than $j$ (\cref{lem:Bx}). Hence we have
	\begin{align*}
		\hspace{-10mm}\Prob{I_t =j, j< \ist|\Ht} = &~ \Prob{j \succ_t k, \forall k>j, j < i^\star|\Ht} \le \Prob{j \succ_t k, \forall k \ge \ist, j < i^\star|\Ht}.
		\intertext{Since the feedback from an arm is independent of the feedback of other arms,}
		=&~ \Prob{j \succ_t i^\star , j < i^\star|\Ht} \Prob{j \succ_t k, \forall k > \ist, j < i^\star|\Ht}.
		\intertext{If arm $j$ is preferred over the arm $i^\star$ then $\opijsts < C_{i^\star} - C_j$. As $C_{i^\star} - C_j = \pijts - \xi_j$ for $j<i^\star$,}
		=&~ \Prob{ \opijsts < \pijts - \xi_j|\Ht}  \Prob{j \succ_t k, \forall k > \ist, j < i^\star|\Ht} \\
		=&~ \left(1 - \Prob{ \opijsts \ge \pijts - \xi_j|\Ht}\right)  \Prob{j \succ_t k, \forall k > \ist, j < i^\star|\Ht} \\
		\implies \Prob{I_t=j, j< \ist|\Ht}  \le &~  (1 - q_{j,t})  \Prob{j \succ_t k, \forall k > \ist, j < i^\star|\Ht}. \mbox{\hspace{4mm} (\cref{def:q})} \numberthis \label{eq:lowerOptimal}
	\end{align*}
	Similarly, the probability of selecting an arm whose index is larger than the optimal arm can be lower bounded as follows:
	\begin{align*}
		\Prob{I_t \ge i^\star|\Ht} &\ge \Prob{I_t = i^\star|\Ht} \ge \Prob{I_t = i^\star, \ist \succ_t j, j < \ist|\Ht}\\
		&= \Prob{i^\star \succ_t k, \forall k > \ist,  \ist \succ_t j, j < \ist|\Ht} \mbox{\hspace{4mm} (\cref{lem:Bx})} \\
		&\ge \Prob{\ist \succ_t j, j \succ_t k, \forall k > \ist, j < \ist|\Ht} \mbox{\hspace{5mm} (\cref{def:trans_prop})} \\
		&=  \Prob{\ist \succ_t j, j < \ist|\Ht}  \Prob{j \succ_t k, \forall k > \ist, j < \ist|\Ht}.
		\intertext{If arm $\ist$ is preferred over the arm $j$ then $\opijsts \ge C_{i^\star} - C_j$. As $C_{i^\star} - C_j = \pijts - \xi_j$ for $j<i^\star$,}
		&= \Prob{ \opijsts \ge \pijts - \xi_j|\Ht} \Prob{j \succ_t k, \forall k > \ist, j < i^\star}\\
		\implies \Prob{I_t \ge \ist|\Ht} & \ge q_{j,t} \Prob{j \succ_t k, \forall k > \ist, j < i^\star}. \mbox{\hspace{17mm} (\cref{def:q})} \numberthis \label{eq:optimalLower}
	\end{align*}
	Combining the \cref{eq:lowerOptimal} and \cref{eq:optimalLower}, we get
	\begin{equation*}
		\Prob{I_t=j, j< \ist|\Ht} \le \frac{(1-q_{j,t})}{q_{j,t}} \Prob{I_t \ge \ist|\Ht}. \qedhere
	\end{equation*}
\end{proof}

\begin{restatable}{lem}{probLow}
	\label{lem:probLow}
	Let $P \in \PWD$ and satisfies the transitivity property. If $s$ be the number of times the sub-optimal arm $j$ is selected by \ref{alg:TS_USS} then, for any $j < i^\star$,
	\begin{equation*}
		\sum_{t=1}^T\Prob{I_t=j, j< \ist} \le \frac{24}{\xi_j^2} + \sum_{s \ge 8/\xi_j} \Theta\left(\exp^{-{s\xi_j^2}/{2}} + \frac{\exp^{-{sd( \pijs - \xi_j,\pijs)}}}{(s+1)\xi_j^2} + \frac{1}{\exp^{{s\xi_j^2}/{4}} - 1} \right).
	\end{equation*}
\end{restatable}
\begin{proof}\textbf{(sketch)}
	Using \cref{lem:relationLowerOptimal} and property of conditional expectations, we can have $\sum_{t=1}^T \Prob{I_t=j, j< \ist} = \sum_{t=1}^T \EE{ \Prob{jI_t=j, j< \ist|\Ht}}$. By using some simple algebraic manipulations on quantity $\sum_{t=1}^T  \EE{ \Prob{I_t=j, j< \ist|\Ht}}$ with \cref{lem:betaExpBound}, we can get the above stated upper bound.
\end{proof}

The detailed proof of \cref{lem:probLow} and all other missing proofs appear in the supplementary material. Our next result is useful to bound the probability by which \ref{alg:TS_USS} prefers the sub-optimal arms whose index is larger than the optimal arm.

\begin{lem}
	\label{lem:probHighPart2}
	Let $\hpijst$ be the empirical estimate of $\pijs$ and $j>\ist$. Then, for any $x_j > \pijs$ and $y_j > x_j$,
	\begin{equation*}
		\sum_{t=1}^T \Prob{\hpijst \le x_j, \opijsts > y_j} \le \frac{\ln T}{d(x_j, y_j)} + 1.
	\end{equation*}
\end{lem}

\begin{proof}
	Define $L_j(T) = \frac{\ln T}{d(x_j, y_j)}$. Let $N_j(t)$ be the number of times the output from arm $j$ is observed in $t$ rounds. Then, the given probability term can be decomposed into two parts:
	\begin{align*}
	\sum_{t=1}^T \Prob{\hpijst \le x_j, \opijsts > y_j} &= \sum_{t=1}^T \Prob{\hpijst \le x_j, \opijsts > y_j, N_j(t) \le L_j(T)} + \\
	&\qquad\qquad  \sum_{t=1}^T \Prob{\hpijst \le x_j, \opijsts > y_j, N_j(t) > L_j(T)} \\
	\le  L_j(T)& + \sum_{t=1}^T \Prob{\hpijst \le x_j, \opijsts > y_j, N_j(t) > L_j(T)}. \label{eq:decomUB} \numberthis 
	\end{align*}
	The first term of the above decomposition is bounded trivially by $L_j(T)$. To bound the second term, we demonstrate that if $N_j(t)$ is large enough and event $\hpijst \le x_j$ is satisfied, then the probability that the event $\opijtst > y_j$ happens, is small. Then,
	\begin{align*}
		\sum_{t=1}^T &\Prob{\hpijst \le x_j, \opijsts > y_j, N_j(t) > L_j(T)} \\
		&\qquad= \sum_{t=1}^T \EE{\one{\hpijst \le x_j, \opijsts > y_j, N_j(t) > L_j(T)}} \\
		&\qquad = \EE{ \sum_{t=1}^T \EE{\one{\hpijst \le x_j, \opijsts > y_j, N_j(t) > L_j(T)}| \Ht}}.
		\intertext{Since $N_j(t)$ and $\hpijst$ are determined by the history $\Ht$,}
		&\qquad = \EE{ \sum_{t=1}^T \one{\hpijst \le x_j, N_j(t) > L_j(T)} \Prob{\opijsts > y_j| \Ht} }. \numberthis \label{eq:secondTerm}
	\end{align*}
	Now, by definition, $\mathcal{S}_{i^\star j}(t) = \hpijst N_j(t)$, and therefore, $\opijsts$ is a Beta$(\hpijst N_j(t) + 1, (1-\hpijst)N_j(t) + 1)$ distributed random variable. A Beta$(\alpha, \beta)$ random variable is stochastically dominated by Beta$(\alpha^\prime, \beta^\prime)$ if $\alpha^\prime \ge \alpha, \beta^\prime \le \beta$. Therefore, if $\hpijst \le x_j$, the distribution of $\opijsts$ is stochastically dominated by Beta$(x_j N_j(t) + 1, (1-x_j)N_j(t))$. Therefore, given a history $\Ht$ such that $\hpijst \le x_j$ and $N_j(t) > L_j(T)$, we have
	\begin{equation*}
		\Prob{\opijsts > y_j| \Ht} = 1 - F_{x_j N_j(t) + 1, (1-x_j)N_j(t)}^{beta}(y_j).
	\end{equation*}
	Now, using Beta-Binomial equality (\cref{fact:beta_binomial}), we obtain that for any fixed $N_j(t) > L_j(T)$,
	\begin{align*}
		1 - F_{x_j N_j(t) + 1, (1-x_j)N_j(t)}^{beta}(y_j) &= F_{N_j(t), y_j}^{B}(x_j N_j(t)) &\mbox{(using \cref{fact:beta_binomial})}
	\end{align*}
	Here $F_{N_j(t), y_j}^{B}(x_j N_j(t))$ is the cdf of Binomial distribution with parameter $y_j$ and $N_j(T)$ observations. Let $\mathcal{S}_t^\prime$ be the number of successes observed in $N_j(T)$ observations. Then,
	\begin{align*}
		1 - F_{x_j N_j(t) + 1, (1-x_j)N_j(t)}^{beta}(y_j) &= \Prob{\mathcal{S}_t^\prime \le x_j N_j(t)} \\
		&= \Prob{\frac{\mathcal{S}_t^\prime}{N_j(t)} \le x_j}\\
		&= \Prob{\hat{y}_j \le x_j} \hspace{17mm} \mbox{(using $\hat{y}_j = \mathcal{S}_t^\prime/N_j(t)$)}\\
		&\le \exp^{-N_j(t)d(x_j, y_j)} \hspace{10mm} \mbox{(using Chernoff-Hoeffding bound)}\\
		&\le \exp^{-L_j(t)d(x_j, y_j)}, \hspace{10mm} \mbox{(as $N_j(t) > L_j(T)$)}
	\end{align*}
	which is smaller than $1/T$ because $L_j(T) = \frac{\log(T)}{d(x_j, y_j)}$. Substituting, we get that for a history $\Ht$ such that $\hpijst \le x_j$ and $N_j(t) > L_j(T)$,
	\begin{align*}
		\Prob{\opijsts > y_j| \Ht} \le \frac{1}{T}.
	\end{align*}
	For other history $\Ht$, the indicator term $\one{\hpijst \le x_j, N_j(t) > L_j(T)}$ in \cref{eq:secondTerm} will be 0 as either event $\hpijst \le x_j$ or event $N_j(t) > L_j(T)$ is violated. Summing over $t$, this bounds the right hand side term in \cref{eq:secondTerm} as follows:
	\begin{align*}
		\sum_{t=1}^T \Prob{\hpijst \le x_j, \opijsts > y_j, N_j(t) > L_j(T)} &\le \EE{ \sum_{t=1}^T \frac{\one{\hpijst \le x_j, N_j(t) > L_j(T)}}{T} } \\
		&\le \EE{ \sum_{t=1}^T \frac{1}{T} } \\
		&= 1.
	\end{align*}
	Replacing the second term in \cref{eq:decomUB} by its upper bound and $L_j(T)$ with its value,
	\begin{equation*}
		\sum_{t=1}^T \Prob{\hpijst \le x_j, \opijsts > y_j} \le \frac{\ln T}{d(x_j, y_j)} + 1. \qedhere
	\end{equation*}
\end{proof}

\begin{restatable}{lem}{probHighPartOne}
	\label{lem:probHighPart1}
	For any $x_j > \pijs$,
	\begin{equation*}
		\sum_{t=1}^T \Prob{\hpijst > x_j} \le \frac{1}{d(x_j, \pijs)}.
	\end{equation*}
\end{restatable}
\begin{proof}\textbf{(sketch)}
	This result is easily proved by using Chernoff-Hoeffding bound. See details in the supplementary material.
\end{proof}

\begin{restatable}{lem}{probHigh}
	\label{lem:probHigh}
	Let $P \in \PWD$. For any $\epsilon > 0$ and $j > i^\star$,
	\begin{equation*}
		\sum_{t=1}^T\Prob{j \succ_t i^\star, j> \ist} \le (1+\epsilon)\frac{\ln T}{d(\pijs, \pijs + \xi_j)} + O\left(\frac{1}{\epsilon^2}\right).
	\end{equation*}
\end{restatable}
\begin{proof}\textbf{(sketch)}
	Let $\pijs < x_j < y_j < \pijs + \xi_j$ where $j>\ist$. Then, it can be easily shown that $\sum_{t=1}^T\Prob{j \succ_t i^\star, j > \ist} \le  \sum_{t=1}^T \Prob{\hpijst \le x_j, \opijsts > y_j} + \sum_{t=1}^T \Prob{\hpijst > x_j}.$ The upper bound on first term of right hand side quantity is given by \cref{lem:probHighPart2} and the upper bound of the second term of right hand side quantity is given by \cref{lem:probHighPart1}.  Then, for $\epsilon \in (0,1)$ with suitable values of $x_j$ and $y_j$, we can get the above stated upper bound.
\end{proof}

Let $\Delta_j = C_j + \gamma_j - (C_{i^\star} + \gamma_{i^\star})$ be the sub-optimality gap for arm $j$. Now we state the problem dependent regret upper bound of \ref{alg:TS_USS}.
\begin{restatable}[Problem Dependent Bound]{thm}{depRegretBound}
	\label{thm:depRegretBound}
	Let $P \in \PWD$ and satisfies the transitivity property. If $\epsilon > 0$ then, the expected regret of \ref{alg:TS_USS} in $T$ rounds is bounded by
	\begin{equation*}
		{\Regret_T} \le \sum_{j > i^\star} \frac{(1+\epsilon)\ln T}{d(\pijs, \pijs + \xi_j)}\Delta_j + O\left(\frac{K-i^\star}{\epsilon^2}\right),
	\end{equation*}
\end{restatable}
\begin{proof}\textbf{(sketch)}
	Let $M_j(T)$ is the number of times arm $j$ is selected by \ref{alg:TS_USS}. Then, the regret of \ref{alg:TS_USS} is given by ${\Regret_T} = \sum_{j \in [K]}\EE{M_j(T)}\Delta_j = \sum_{j \in [K]}\sum_{t=1}^{T}\EE{\one{I_t = j}}\Delta_j = \sum_{j \in [K]}\sum_{t=1}^{T}\Prob{I_t = j}\Delta_j$. We divide the regret into two parts and it can be re-written as ${\Regret_T} \le \sum_{j < i^\star} \sum_{t=1}^{T} \Prob{I_t=j, j < i^\star} \Delta_j + \sum_{j > i^\star}\sum_{t=1}^{T}\Prob{I_t=j, j > i^\star} \Delta_j$. The first part of the regret is upper bounded by using  \cref{lem:probLow}. For the second part, when arm $I_t > \ist$  is selected, then there exists at least one arm $k > \ist$, which must be preferred over $\ist$. Using transitivity property and a recursive argument, we can show that the selected arm is preferred over the optimal arm.  Hence, $\sum_{j > i^\star}\sum_{t=1}^{T}\Prob{I_t=j, j > i^\star} \Delta_j$ can be upper bounded by $\sum_{j > i^\star} \sum_{t=1}^{T}$ $\Prob{j \succ_t \ist, j > i^\star} \Delta_j$. We can upper bound $\sum_{j > i^\star} \sum_{t=1}^{T}\Prob{j \succ_t \ist, j > i^\star} \Delta_j$ by using \cref{lem:probHigh} to get the above stated regret upper bound for \ref{alg:TS_USS}.
\end{proof}

\noindent
Next we present problem independent bounds on the regret of \ref{alg:TS_USS}.
\begin{restatable}[Problem Independent Bound]{thm}{indepRegretBound}
	\label{thm:indepRegretBound}
	Let $P \in \PWD$ and satisfies the transitivity property. Then the expected regret of \ref{alg:TS_USS} in $T$ rounds 
	\begin{itemize}
		\item for any instance in $\PSD$ is bounded as
		\begin{align*}
			{\Regret_T} \le O\left( \sqrt{KT\ln T}\right).
		\end{align*}
		\item 	for any instance in $\PWD$ is bounded as
		\begin{align*}
			{\Regret_T} \le O\left( \left(K\ln T\right)^{1/3}T^{2/3}\right).
		\end{align*}
	\end{itemize}
\end{restatable}
\begin{proof}\textbf{(sketch)}
	To get the above problem independent regret upper bound, we maximize the problem-dependent regret of \ref{alg:TS_USS} with respect to the value of $\xi_j$. 
\end{proof}

\begin{cor}
	Let $P \in \PWD$ and satisfies the transitivity property. Then the expected regret of \ref{alg:TS_USS} on $\PSD$ is $\tilde{O}(T^{1/2})$ and on $\PWD$ it is $\tilde{O}(T^{2/3})$, where $\tilde{O}$ hides $K$ and the logarithmic terms that are having $T$ in them.
\end{cor}

\noindent
\paragraph{Discussion on optimality of \ref{alg:TS_USS}:} Stochastic partial monitoring problems can be classified as an `easy,' `hard,' or `hopeless' problem with expected regret bounds of the order $\Theta(T^{1/2}), \Theta(T^{2/3})$, or $\Theta(T)$, respectively. And there exists no other class of problems in between \citep{MOR14_bartok2014partial}. The class $\PSD$ is {regret equivalent} to a stochastic multi-armed bandit with side observations \citep{AISTATS17_hanawal2017unsupervised},  for which regret scales as $\Theta(T^{1/2})$, hence $\PSD$ resides in the easy class and our bound on it is near-optimal.  Since $\PWD \supsetneq \PSD$, $\PWD$ is not easy problem. Since $\PWD$ is also learnable, it cannot be a hopeless problem. Therefore, the class $\PWD$ is hard. We thus conclude that the regret bound of \ref{alg:TS_USS} is also near-optimal in $T$ up to a logarithmic term.

%% file: experiment.tex

We evaluate the performance of \ref{alg:TS_USS} on different problem instances derived from synthetic and two real datasets: PIMA Indians Diabetes \citep{UCI16_pima2016kaggale} and Heart Disease (Cleveland) \citep{HEART98_robert1988va}. The details of the used problem instances are given as follows.

\paragraph{Synthetic Dataset:} 
We generate synthetic Bernoulli Symmetric Channel (BSC) dataset \citep{AISTATS17_hanawal2017unsupervised} as follows: The true binary feedback $Y_t$ is generated from i.i.d. Bernoulli random variable with mean $0.7$. The problem instance used in the experiment has three arms. We fix feedback as true binary feedback for the first arm with probability $0.6$, second arm with probability $0.7$, and third arm with probability $0.8$. To ensure strong dominance, we impose the condition during data generation. When the feedback of arm $1$ matches the true binary feedback, we introduce error up to 10\%  to the feedback of arm $2$ and $3$. We use five problem instances of the BSC dataset by varying the cumulative cost of playing the arms as given in \cref{table:bsc}.  

\begin{table}[H]
	\centering
	\scriptsize
	\setlength\tabcolsep{10pt}
	\setlength\extrarowheight{3pt}
	\begin{tabular}{ |c|c|c|c|c|} 
		\hline
		\multirow{1}{*}{\bf Values/ \newline Arms}&Arm $1$ &Arm $2$&Arm $3$& \multirow{1}{*}{\parbox{1cm}{~\\ $\WD$ Property}}\\ 
		\cline{1-4}
		Error-rate $(\gamma_i)$ & 0.3937 & 0.2899 & 0.1358 &  \multicolumn{1}{c|}{~}  \\
		\hline
		Instance 1 Costs& \textcolor{red}{\textbf{0.05}} & 0.285          & 0.45        &  \multicolumn{1}{c|}{\checkmark}  \\ 
		\hline
		Instance 2 Costs& 0.05         & \textcolor{red}{\textbf{0.1}} & 0.53      & \multicolumn{1}{c|}{\checkmark} \\ 
		\hline
		Instance 3 Costs& \textcolor{red}{\textbf{0.05  }} & 0.3           & 0.45      &    \multicolumn{1}{c|}{\checkmark}  \\ 
		\hline
		Instance 4 Costs& 0.05         & 0.25         & \textcolor{red}{\textbf{0.29}} &  \multicolumn{1}{c|}{\checkmark}\\ 
		\hline
		Instance 5 Costs& 0.1        & \textcolor{red}{\textbf{0.2}} & 0.41       &  \multicolumn{1}{c|}{\ding{53}} \\
		\hline
	\end{tabular}
	\caption{\small $\WD$ propoerty doesn't hold for Instance 5. Optimal arm's cost is in \textcolor{red}{\bf red bold} font.}
	\label{table:bsc}
\end{table}

\paragraph{Real Datasets:} An arm $i$ represents a classifier whose prediction is treated as the feedback of the arm $i$. The disagreement label for $(i,j)$ pair is computed using the labels of classifier (Clf.) $i$ and $j$. In Heart Disease dataset, each sample has $12$ features. We split the features into three subsets and train a logistic classifier on each subset. We associate 1st classifier with the first $6$ features as input, including cholesterol readings, blood sugar, and rest-ECG. The 2nd classifier, in addition to the $6$ features, utilizes the thalach, exang and oldpeak features, and the 3rd classifier uses all the features. In PIMA Indians Diabetes dataset, each sample has $8$ features related to the conditions of the patient. We split the features into three subsets and train a logistic classifier on each subset. We associate 1st classifier with the first $6$ features as input. These features include patient profile. The 2nd classifier, in addition to the $6$ features, utilizes the feature on the glucose tolerance test, and the 3rd classifier uses all the previous features and the feature that gives values of insulin test. The PIMA Indians Diabetes dataset has $768$ samples, whereas the Heart Disease dataset has only $297$ samples. As $10000$ rounds are used in our experiments, we select a sample from the original dataset in a round-robin fashion and give it as input to the algorithm. The details about the different costs used in five problem instances of the real datasets are given in \cref{table:real_datasets}.

\begin{table}[!h]
	\centering
	\scriptsize
	\setlength\tabcolsep{10pt}
	\setlength\extrarowheight{3pt}
	\begin{tabular}{|c|c|c|c|c|c|c|c|}
		\hline
		\multirow{2}{*}{\parbox{2.7cm}{\bf Values/ Classifiers (Arms)}} &\multicolumn{3}{|c|}{\bf PIMA Indians Diabetes}&\multicolumn{3}{|c|}{\bf Heart Disease} &\multirow{3}{*}{\parbox{1cm}{$\WD$ Property}}\\ 
		\cline{2-7} 
		&Clf. 1 &Clf. 2&Clf. 3&Clf. 1 &Clf. 2&Clf. 3&\\ \cline{1-7}
		Error-rate ($\gamma_{i}$) & 0.3098 &0.233&0.2278&0.2929 &0.2025& 0.1483 &\\ 
		\hline
		Instance  1 Costs& \textcolor{red}{\textbf{0.05}}& 0.28& 0.45&\textcolor{red}{\textbf{0.02}}& 0.32& 0.45&\multicolumn{1}{|c|}{\checkmark}\\ 
		\hline
		Instance  2 Costs& 0.2& \textcolor{red}{\textbf{0.25}}& 0.269&0.2& \textcolor{red}{\textbf{0.25}}& 0.395&\multicolumn{1}{|c|}{\checkmark}\\ 
		\hline
		Instance  3 Costs& \textcolor{red}{\textbf{0.05}}& 0.309& 0.45&\textcolor{red}{\textbf{0.02}}& 0.34& 0.45&\multicolumn{1}{|c|}{\checkmark}\\ 
		\hline
		Instance  4 Costs& 0.2& 0.25& \textcolor{red}{\textbf{0.255}}&0.2& 0.25& \textcolor{red}{\textbf{0.3}}&\multicolumn{1}{|c|}{\checkmark}\\ 
		\hline
		Instance  5 Costs& \textcolor{red}{\textbf{0.05}}& 0.146& 0.3&0.2& \textcolor{red}{\textbf{0.25}}& 0.325&\multicolumn{1}{|c|}{\ding{53}}\\ 
		\hline
	\end{tabular}
	\caption{Costs of different problem instances which are derived from real datasets. $\WD$ property doesn't hold for Instance 5 and cost of optimal arm is in  \textcolor{red}{\bf red bold} font.}
	\label{table:real_datasets}
\end{table}

\paragraph{Verifying $\WD$ property:} The error-rate associated with each arm is known to us as given in  \cref{table:bsc} and \cref{table:real_datasets} (but note that the error-rates are unknown to the algorithm); hence we can find an optimal arm for a given problem instance. After knowing optimal arm, $\WD$ property is verified by using the disagreement probability estimates after $10000$ rounds.

\subsection{Experimental Results} 
We fix the time horizon to $10000$ in all experiments and repeat each experiment $500$ times. The average regret is presented with a $95$\% confidence interval. The vertical line on each plot shows the confidence interval.
\begin{figure}[!h]
	\subfigure[\small BSC Dataset]{
		\label{fig:bsc}
		\includegraphics[scale=0.31425]{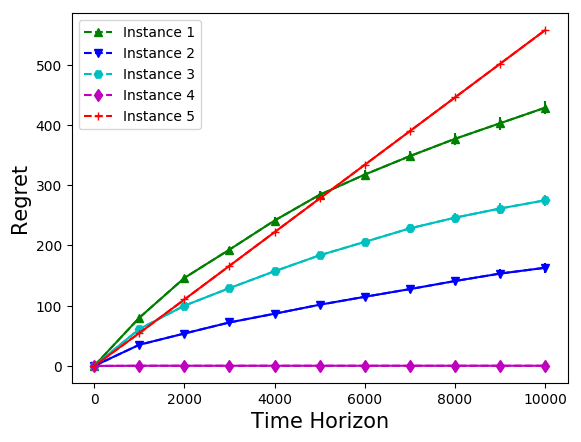}
	}
	\subfigure[\small PIMA Indians Diabetes]{
		\label{fig:pima}
		\includegraphics[scale=0.31425]{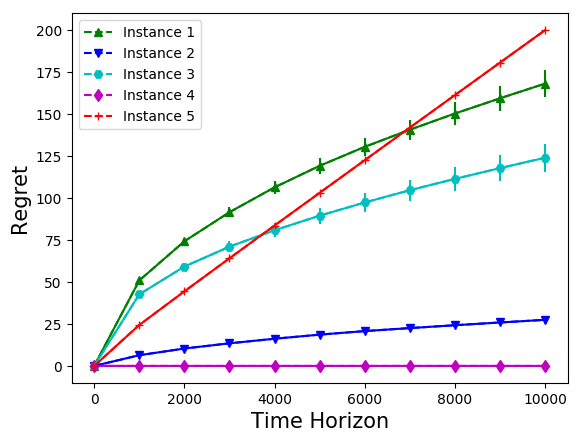}
	}
	\subfigure[\small Heart Disease]{
		\label{fig:heart}
		\includegraphics[scale=0.31425]{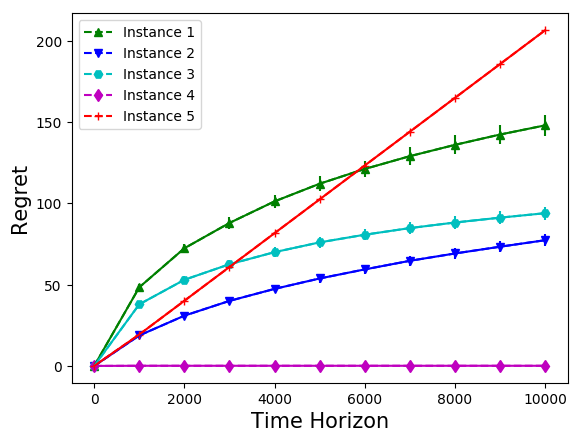}
	}
	\caption{\small Regret of \ref{alg:TS_USS} for different problem instances derived from synthetic and real datasets.}
	\label{fig:regret}
\end{figure}

\paragraph{Expected Cumulative Regret v/s Time Horizon:} The {\it Regret} of \ref{alg:TS_USS} versus {\it Time Horizon} plots for the different problem instances derived from BSC Dataset and two real datasets are shown in Figure \ref{fig:regret}. These plots verify that any instance that satisfies $\WD$ property has sub-linear regret. Note that \ref{alg:TS_USS} has linear regret for the Instance $5$ as it does not satisfy $\WD$ property.
We also compare the performance of \ref{alg:TS_USS} with existing UCB based algorithm USS-UCB algorithm of \cite{AISTATS19_verma2019online} with value of $\alpha=0.5$ (best possible parameter value mentioned in the paper) and Algorithm 2 of \cite{AISTATS17_hanawal2017unsupervised} with value of $\alpha=1.5$ (as used in the paper) on Heart Disease and PIMA Indians Diabetes datasets. As expected, \ref{alg:TS_USS} outperforms other algorithms with large margins as shown in \cref{fig:dCompare} (PIMA Indians Diabetes dataset) and \cref{fig:hCompare} (Heart Disease dataset). 
\begin{figure}[!h]
	\subfigure[\small PIMA Indians Diabetes]{
		\label{fig:dCompare}
		\includegraphics[scale=0.3114]{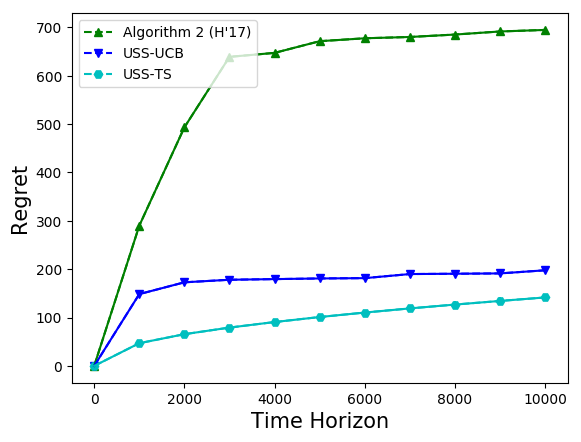}
	}
	\subfigure[\small Heart Disease]{
		\label{fig:hCompare}
		\includegraphics[scale=0.3114]{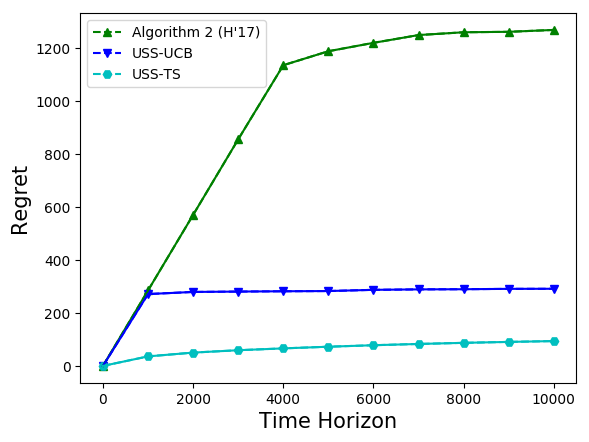}
	}
	\subfigure[\small BSC Dataset]{
		\label{fig:xi}
		\includegraphics[scale=0.3114]{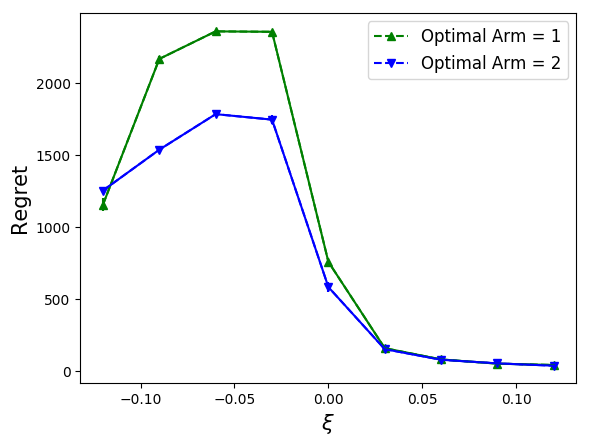}
	}
	\caption{\small Comparing regret of \ref{alg:TS_USS} with USS-UCB \citep{AISTATS19_verma2019online} and Algorithm 2 \citep{AISTATS17_hanawal2017unsupervised} for real datasets (\cref{fig:dCompare} and \cref{fig:hCompare}). Regret behavior of \ref{alg:TS_USS} versus WD property for BSC Dataset is shown in \cref{fig:xi}.}
	\label{fig:allExp}
\end{figure}

\paragraph{Learnability v/s $\WD$ Property:} We experiment with different problem instances of the BSC dataset to know the relationship between regret of \ref{alg:TS_USS} and $\WD$ property. We fixed an optimal arm and vary the cumulative cost of using arms in such a way that we pass from the case where $\WD$ property does not hold ($\xi \le 0$ or $C_j - C_{i^\star} \in (\gamma_{i^\star} - \gamma_j, p_{i^\star j}]$  for any $j>i^\star$ where $\xi := \min_{j>i^\star} \xi_j$) to the situation where $\WD$ property holds $(\xi > 0$). 
When $\WD$ property does not hold for any problem instance, \ref{alg:TS_USS} treats a sub-optimal arm as the optimal arm. In such problem instances, as $C_j - C_{i^\star}$ increases, the regret will also increase due to selection of sub-optimal arm by \ref{alg:TS_USS} until $\WD$ property does not satisfy for that problem instance. When $\WD$ property does not satisfy for a problem instance then $C_j - C_{i^\star} \in (\gamma_{i^\star} - \gamma_j, p_{i^\star j}]$ holds in such cases, hence, it is easy to verify that $\xi$ can not be smaller than $-\max(p_{i^\star j}- (\gamma_{i^\star} - \gamma_j))$.

We consider the problem instances with the minimum possible value of $\xi$ for which problem instance satisfies $\WD$ property. Then we increase the value of $\xi$ by increasing the cumulative cost of the arm. The regret versus $\xi$ plots for BSC Dataset is shown in \cref{fig:xi}. It can be observed that there is a transition at $\xi = 0$. Through our experiments, we show that the stronger the $\WD$ property (large value of $\xi$) for the problem instance, it is easier to identify the optimal arm and, hence the less regret is incurred by \ref{alg:TS_USS}.

%% file: appendix.tex

\section{Useful results needed to prove regret bounds of \ref{alg:TS_USS}}
We use the following results in our proofs.

\begin{fact}[Chernoff bound for Bernoulli distributed random variables]
	\label{fact:chernoff}
	Let $X_1, \ldots, X_n$ be i.i.d. Bernoulli distributed random variables. Let $\hat{\mu}_n=\frac{1}{n}\sum_{i=1}^{n}X_i$ and $\mu = \EE{X_i}$. Then, for any $\epsilon \in (0, 1 - \mu)$,
	\begin{equation*}
		\Prob{\hat{\mu}_n \ge \mu + \epsilon} \le \exp\left( -d(\mu + \epsilon, \mu)n\right),
	\end{equation*}
	and, for any $\epsilon \in (0, \mu)$,
	\begin{equation*}
		\Prob{\hat{\mu}_n \le \mu - \epsilon} \le \exp\left( -d(\mu - \epsilon, \mu)n\right),
	\end{equation*}
	where $d(x, \mu)= x\log\left(\frac{x}{\mu}\right) + (1-x)\log\left(\frac{1-x}{1-\mu}\right)$.
\end{fact}
\noindent
See Section 10.1 of Chapter 10 of book `Bandit Algorithms' \citep{BOOK_lattimorebandit} for proof.

\begin{fact}[Pinsker's Inequality for Bernoulli distributed random variables]
	\label{fact:pinsker}
	For $p,q \in (0,1)$, the KL divergence between two Bernoulli distributions is bounded as:
	\begin{equation*}
		d(p,q) \ge 2(p-q)^2.
	\end{equation*}
\end{fact}

\begin{fact}
	\label{fact:expUB}
	Let $x > 0$ and $D > 0$. Then, for any $a \in (0,1)$,
	\begin{equation*}
		\frac{1}{\exp^{Dx} - 1} \le 
		\begin{cases}
			\frac{\exp^{-Dx}}{1-a} & \left(x \ge \ln\left(1/a\right)/D \right)\\
			\frac{1}{Dx} & \left(x < \ln\left(1/a\right)/D \right). 
		\end{cases}
	\end{equation*}
	Further, we have,
	\begin{equation*}
		\sum_{x=1}^{n} \frac{1}{\exp^{Dx} - 1} \le \Theta\left( \frac{1}{D^2} + \frac{1}{D}\right).
	\end{equation*}
\end{fact}

\begin{proof}
	Using $\exp^y \ge y + 1$ (by Taylor Series expansion), we have $\frac{1}{\exp^{Dx} - 1} \le \frac{1}{Dx}$ as $\exp^{Dx} - 1 \ge Dx$. We can re-write, $\frac{1}{\exp^{Dx} - 1} = \frac{\exp^{-Dx}}{1-\exp^{-Dx}}$. Since $\exp^{-Dx}$ is strictly decreasing function for all $Dx>0$, it is easy to check that $\exp^{-Dx} \le a$ holds for any $x \ge \ln\left(1/a\right)/D$ and $a \in (0,1)$. Hence, $\frac{\exp^{-Dx}}{1-\exp^{-Dx}} \le \frac{\exp^{-Dx}}{1-a}$ for all $x \ge \ln\left(1/a\right)/D$.
	
	~\\
	Now we will prove the second part,
	\begin{align*}
		\sum_{x=1}^{n} \frac{1}{\exp^{Dx} - 1} &\le \frac{\ln(1/a)}{D^2} + \sum_{x\ge \ln\left(1/a\right)/D}^{n} \frac{\exp^{-Dx}}{1-a}  \\
		 &\le \frac{\ln(1/a)}{D^2} + \frac{1}{(1-a)} \int_{x = 0}^{\infty} \exp^{-Dx} dx \\
		 &= \frac{\ln(1/a)}{D^2} + \frac{1}{(1-a)}  \left.\left(\frac{\exp^{-Dx}}{-D}\right)\right|_{x = 0}^{\infty} \\
		 &= \frac{\ln(1/a)}{D^2} + \frac{1}{(1-a)}  \left( 0 - \frac{\exp^{0}}{-D}\right) \\
		 &= \frac{\ln(1/a)}{D^2} + \frac{1}{(1-a)D}\\
		 \implies \sum_{x=1}^{n} \frac{1}{\exp^{Dx} - 1} &\le \Theta\left( \frac{1}{D^2} + \frac{1}{D}\right). \qedhere
	\end{align*}
\end{proof}

\begin{fact}
	\label{fact:diffKL}
	Let $\epsilon \in (0,1)$ and $0< x< y < z < 1$. If $d(y,z) = d(x,z)/(1+\epsilon)$ then
	\begin{equation*}
		y-x \ge \frac{\epsilon}{1 + \epsilon} \cdot \frac{d(x,z)}{\ln\left( \frac{z(1-x)}{x(1-z)}\right)}.
	\end{equation*}
\end{fact}

\begin{proof}
	By definition
	\begin{align*}
		d(p,q) &= p\ln\frac{p}{q} + (1-p)\ln\left(\frac{1-p}{1-q}\right)\\
		&=\ln\left(\left(\frac{p}{q}\right)^p \left(\frac{1-p}{1-q}\right)^{1-p} \right) \\
		&= \ln\left(\left(\frac{q(1-p)}{p(1-q)}\right)^{-p} \right) + \ln\left(\frac{1-p}{1-q}\right)\\
		\implies d(p,q) &= -p\ln\left(\frac{q(1-p)}{p(1-q)}\right) + \ln\left(\frac{1-p}{1-q}\right).
	\end{align*}
	Set $l(p,q) = \ln\left(\frac{q(1-p)}{p(1-q)}\right)$. Note that $l(p,\cdot)$ is a strictly decreasing function of $p$ and positive for all $p < q$. We can re-arrange above equation as
	\begin{equation*}
		p\cdot l(p,q) = -d(p,q) + \ln\left(\frac{1-p}{1-q}\right).
	\end{equation*}
	Using above equation, we have
	\begin{align*}
		y\cdot l(y,z) - x\cdot l(x,z) &= -d(y,z) +  \ln\left(\frac{1-y}{1-z}\right) + d(x,z) -  \ln\left(\frac{1-x}{1-z}\right).
		\intertext{Using $d(y,z) = d(x,z)/(1+\epsilon)$,}
		y\cdot l(y,z) - x\cdot l(x,z) &= \frac{\epsilon}{1+\epsilon}d(x,z) +  \ln\left(\frac{1-y}{1-x}\right).
		\intertext{After adding $y(l(x,z) - l(y,z))$ both side, we have }
		(y - x) l(x,z) &= \frac{\epsilon}{1+\epsilon}d(x,z) +  \ln\left(\frac{1-y}{1-x}\right) + y(l(x,z) - l(y,z)).
		\intertext{Using $l(x,z) = \ln\left(\frac{z(1-x)}{x(1-z)}\right)$ and $l(y,z)= \ln\left(\frac{z(1-y)}{y(1-z)}\right)$}
		&= \frac{\epsilon}{1+\epsilon}d(x,z) +  \ln\left(\frac{1-y}{1-x}\right) + y \ln\left(\frac{y(1-x)}{x(1-y)}\right) \\
		&= \frac{\epsilon}{1+\epsilon}d(x,z) +  \ln\left(\left( \frac{y(1-x)}{x(1-y)} \right)^y\cdot \frac{1-y}{1-x}\right) \\
		&=\frac{\epsilon}{1+\epsilon}d(x,z) +  \ln\left(\left(\frac{y}{x}\right)^y\left( \frac{1-y}{1-x}\right)^{1-y}\right)\\
		&=\frac{\epsilon}{1+\epsilon}d(x,z) + d(y,x)
		\intertext{As $d(p,q) \ge 0$ and dividing both side by $l(x,z)$,}
		\implies y- x  &\ge \frac{\epsilon}{1+\epsilon}\cdot \frac{d(x,z)}{l(x,z)}.
	\end{align*}
	Substituting value of $l(x,z)$ in the above equation, we get
	\begin{equation*}
		y -x \ge \frac{\epsilon}{1 + \epsilon} \cdot \frac{d(x,z)}{\ln\left( \frac{z(1-x)}{x(1-z)}\right)}. \qedhere
	\end{equation*}	
\end{proof}

\section{Leftover proofs from \cref{sec:ts_uss}}

\probLow*
\begin{proof}
	Applying \cref{lem:relationLowerOptimal} and properties of conditional expectations, we have
	\begin{align*}
		\sum_{t=1}^T \Prob{I_t=j, j < \ist} &= \sum_{t=1}^T  \EE{ \Prob{I_t=j, j < \ist|\Ht}}.
		\intertext{As $q_{j,t}$ is fixed given $\Ht$,}
		\implies \sum_{t=1}^T  \Prob{I_t=j, j < \ist}  &\le \sum_{t=1}^T  \EE{\frac{(1-q_{j,t})}{q_{j,t}}\Prob{I_t \ge i^\star|\Ht}}\\
		&\le \sum_{t=1}^T  \EE{\EE{\frac{(1-q_{j,t})}{q_{j,t}}\one{I_t \ge i^\star}|\Ht}}.
		\intertext{Using law of iterated expectations,}
		\implies \sum_{t=1}^T  \Prob{I_t=j, j < \ist}  &\le  \sum_{t=1}^T  \EE{\frac{(1-q_{j,t})}{q_{j,t}}\one{I_t \ge i^\star}}.  \numberthis \label{eq:lowProbSum}
	\end{align*}
	Let $s_m$ denote the time step at which the output of arm $i^\star$ is observed for the $m^{th}$ time for $m\ge1$, and let $s_0 = 0$. For $j<i^\star$, whenever the output from arm $i^\star$ is observed then the output from arm $j$ is also observed due to the cascade structure. Note that $q_{j,t} = \Prob{\opijsts > \pijs - \xi_j|\Ht}$ changes only when the distribution of $\opijsts$ changes, that is, only on the time step when the feedback from arms $i^\star$ and $j$ are observed. It only happens when selected arm $I_t\ge \ist$. Hence, $q_{j,t}$ is the same at all time steps $t \in \{s_m +1, \ldots, s_{m+1}\}$ for every $m$. Using this fact, we can decompose the right hand side term in \cref{eq:lowProbSum} as follows,
	\begin{align*}
		\sum_{t=1}^T  \EE{\frac{(1-q_{j,t})}{q_{j,t}}\one{I_t \ge i^\star}} &= \sum_{m=0}^{T-1}  \EE{\frac{(1-q_{j,s_m+1})}{q_{j,s_m+1}} \sum_{t=s_m + 1}^{s_{m+1}}\one{I_t \ge i^\star}} \\
		&\le  \sum_{m=0}^{T-1}  \EE{\frac{(1-q_{j,s_m+1})}{q_{j,s_m+1}}} \\
		&=  \sum_{k=0}^{T-1} \EE{\frac{1}{q_{j,s_m+1}} - 1}.
	\end{align*}
	Using above bound in \cref{eq:lowProbSum}, we get
	\begin{equation*}
		\sum_{t=1}^T  \Prob{I_t=j, j < \ist}  \le  \sum_{m=0}^{T-1} \EE{\frac{1}{q_{j,s_m+1}}-1}.
	\end{equation*}
	Substituting the bound from \cref{lem:betaExpBound} with $\mu = \pijs, x = \pijs - \xi_j, \Delta(x) = \xi_j,$ and $q_n(x)= q_{j,s_m}$, we obtain the following bound,
	\begin{equation*}
		\sum_{t=1}^T  \Prob{I_t=j, j < \ist}  \le \frac{24}{\xi_j^2} + \sum_{s \ge 8/\xi_j} \Theta\left(\exp^{-{s\xi_j^2}/{2}} + \frac{\exp^{-{sd( \pijs - \xi_j,\pijs)}}}{(s+1)\xi_j^2} + \frac{1}{\exp^{{s\xi_j^2}/{4}} - 1} \right). \qedhere
	\end{equation*}
	
\end{proof}

\probHighPartOne*
\begin{proof}
	Let $s_m$ denote the time step at which the outputs of arm $i^\star$ and $j$ is observed for the $m^{th}$ time for $m\ge1$, and let $s_0 = 0$. Note that probability $\Prob{\hpijst > x_j}$ changes when the outputs from both arm $i^\star$ and $j$ are observed. Hence, we have
	\begin{align*}
		\sum_{t=1}^T \Prob{\hpijst >x_j} &\le \sum_{m=0}^{T-1} \Prob{\hpijs(s_{m+1}) >x_j}\\
		&= \sum_{m=0}^{T-1} \Prob{\hpijs(s_{m+1}) - \pijs > x_j-\pijs}\\
		&\le \sum_{m=0}^{T-1} \exp^{-kd(\pijs + x_j-\pijs, \pijs)}  \hspace{10mm} \text{(using \cref{fact:chernoff}})\\
		&=\sum_{m=0}^{T-1} \exp^{-kd(x_j, \pijs)}.
	\end{align*}
	Using $\sum_{s\ge 0}\exp^{-sa} \le 1/a$, we get
	\begin{equation*}
	\sum_{t=1}^T \Prob{\hpijst >x_j} \le \frac{1}{d(x_j, \pijs)}. \qedhere
	\end{equation*}
\end{proof}

\probHigh*
\begin{proof}
	Let $\pijs < x_j < y_j < \pijs + \xi_j$ for any $j > i^\star$. Than, 
	\begin{align*}
		\sum_{t=1}^T\Prob{j \succ_t i^\star, j > \ist} &= \sum_{t=1}^T \Prob{\opijsts > \pijs + \xi_j} \\
		&\le \sum_{t=1}^T \Prob{\opijsts > y_j}\\
		& \le  \sum_{t=1}^T \Prob{\hpijst \le x_j, \opijsts > y_j} + \sum_{t=1}^T \Prob{\hpijst > x_j}.
	\end{align*}
	Using \cref{lem:probHighPart1} and \cref{lem:probHighPart2}, we have
	\begin{equation*}
		\sum_{t=1}^T\Prob{j \succ_t i^\star, j > \ist} \le  \frac{\ln T}{d(x_j, y_j)} + 1 + \frac{1}{d(x_j, \pijs)}.
	\end{equation*}
	For $\epsilon \in (0,1)$, we set $x_j \in (\pijs, \pijs + \xi_j)$ such that $d(x_j, \pijs + \xi_j) = d(\pijs, \pijs + \xi_j)/(1+\epsilon)$, and set $y_j \in (x_j, \pijs + \xi_j)$ such that $d(x_j, y_j) = d(x_j, \pijs + \xi_j)/(1+\epsilon) = d(\pijs, \pijs + \xi_j)/(1+\epsilon)^2$. Then this gives
	\begin{equation*}
		\frac{\ln(T)}{d(x_j,y_j)} = (1+\epsilon)^2\frac{\ln(T)}{d(\pijs, \pijs + \xi_j)}.
	\end{equation*}
	Using \cref{fact:diffKL}, if $\epsilon \in (0,1)$, $x_j \in (\pijs, \pijs + \xi_j)$, and $d(x_j, \pijs + \xi_j) = d(\pijs, \pijs + \xi_j)/(1+\epsilon)$ then
	\begin{equation*}
		x_j - \pijs \ge \frac{\epsilon}{1 + \epsilon}. \frac{d(\pijs, \pijs + \xi_j)}{\ln\left( \frac{(\pijs + \xi_j)(1-\pijs)}{\pijs(1-\pijs - \xi_j)}\right)}.
	\end{equation*}
	Using Pinsker's Inequality (\cref{fact:pinsker}), $1/d(x_j, \pijs) \le 1/2(x_j - \pijs)^2 = O({1}/{\epsilon^2})$ where big-Oh is hiding functions of the $\pijs$ and $\xi_j$,
	\begin{align*}
		\sum_{t=1}^T \Prob{j \succ_t \ist, j > i^\star} &\le (1+\epsilon)^2\frac{\ln(T)}{d(\pijs, \pijs + \xi_j)} + O\left(\frac{1}{\epsilon^2}\right)\\
		&\le (1+3\epsilon)\frac{\ln(T)}{d(\pijs, \pijs + \xi_j)} + O\left(\frac{1}{\epsilon^2}\right)\\
		&\le (1+\epsilon^\prime)\frac{\ln(T)}{d(\pijs, \pijs + \xi_j)} + O\left(\frac{1}{{\epsilon^\prime}^2}\right),
	\end{align*}
	where $\epsilon^\prime = 3\epsilon$ and the big-Oh above hides $\pijs$ and $\xi_j$ in addition to the absolute constants. Replacing $\epsilon$ by $\epsilon^\prime$ completes the proof.
\end{proof}

\depRegretBound*
\begin{proof}
	Let $M_j(T)$ is the number of times arm $j$ is selected by \ref{alg:TS_USS}. Than, the regret is
	\begin{align*}
		{\Regret_T} &= \sum_{j \in [K]}\EE{M_j(T)}\Delta_j = \sum_{j \in [K]}\EE{\sum_{t=1}^{T}\one{I_t = j}}\Delta_j \\ 
		&= \sum_{j \in [K]}\sum_{t=1}^{T}\EE{\one{I_t = j}}\Delta_j = \sum_{j \in [K]}\sum_{t=1}^{T}\Prob{I_t = j}\Delta_j \\
		&=  \sum_{j \in [K]}\sum_{t=1}^{T}\Prob{I_t=j, j \ne \ist}\Delta_j \\
		\implies {\Regret_T}&=  \sum_{j < i^\star} \sum_{t=1}^{T}\Prob{I_t=j, j < i^\star} \Delta_j + \sum_{j > i^\star}\sum_{t=1}^{T}\Prob{I_t=j, j > i^\star} \Delta_j \numberthis \label{eq:regretSum}
	\end{align*}
	First, we bound the first of term of summation. From \cref{lem:probLow}, we have
	\begin{equation*}
		\sum_{t=1}^{T}\Prob{I_t =j, j < i^\star} \le \frac{24}{\xi_j^2} + \sum_{s \ge 8/\xi_j} \Theta\left(\exp^{-{s\xi_j^2}/{2}} + \frac{\exp^{-{sd( \pijs - \xi_j,\pijs)}}}{(s+1)\xi_j^2} + \frac{1}{\exp^{{s\xi_j^2}/{4}} - 1} \right).
	\end{equation*}
	Using $\sum_{s\ge 0}\exp^{-sa} \le 1/a$, $d( \pijs - \xi_j,\pijs) \le 2\xi_j^2$ (\cref{fact:pinsker}), and \cref{fact:expUB}, we have
	\begin{align*}
		\sum_{t=1}^T \Prob{I_t =j,  j < i^\star} &\le \frac{24}{\xi_j^2} + \Theta\left(\frac{1}{\xi_j^2} + \frac{1}{\xi_j^4} + \left(\frac{1}{\xi_j^4} + \frac{1}{\xi_j^2} \right) \right) 
		\le O(1). \numberthis \label{equ:lowerRegret}
	\end{align*}
	If arm $I_t > \ist$ is selected then there exists at least one arm $k_1 > \ist$ which must be preferred over $\ist$. If the index of arm $k_1$ is smaller than the selected arm, then there must be an arm $k_2>k_1$, which must be preferred over $k_1$. By transitivity property, arm $k_2$ is also preferred over $\ist$. If the index of arm $k_2$ is still smaller of the selected arm, we can repeat the same argument. Eventually, we can find an arm $k^\prime$ whose index is larger than the selected arm, and it is preferred over arm $k_i, \ldots, k_1, \ist$. Note that the selected arm must be preferred over $k^\prime$; hence the selected arm is also preferred over $\ist$. We can write it as follows:
	\begin{align*}
		\sum_{t=1}^{T}\Prob{I_t=j, j > i^\star} \Delta_j =&~\sum_{t=1}^{T}\Prob{I_t=j, j > i^\star, k^\prime \succ_t k, k \succ_t \ist, k^\prime > j, k > \ist} \Delta_j\\
		=&~\sum_{t=1}^{T}\Prob{I_t=j, j > i^\star, k^\prime \succ_t \ist, k^\prime > j} \Delta_j \mbox{\hspace{2mm} (\cref{def:trans_prop})}\\
		=&~\sum_{t=1}^{T}\Prob{j \succ_t k, \forall k >j, j > i^\star, k^\prime \succ_t \ist, k^\prime > j} \Delta_j \mbox{\hspace{1mm}(\cref{lem:Bx})}\\
		=&~\sum_{t=1}^{T}\Prob{j \succ_t k, \forall k >j, j > i^\star, j \succ_t \ist} \Delta_j \mbox{\hspace{2mm} (\cref{def:trans_prop})}\\
		\implies \sum_{t=1}^{T}\Prob{I_t=j, j > i^\star} \Delta_j \le&~\sum_{t=1}^{T}\Prob{j \succ_t \ist, j > \ist} \Delta_j. \numberthis \label{equ:selectToPrefer}
	\end{align*}
	Using \cref{lem:probHigh} to upper bound $\sum_{t=1}^{T}$ $ \Prob{j \succ_t \ist, j > i^\star} \Delta_j$ and with \cref{equ:lowerRegret}, we get
	\begin{align*}
		{\Regret_T} &\le O (1)+ \sum_{j > i^\star}\left( (1+\epsilon)\frac{\ln(T)}{d(\pijs, \pijs + \xi_j)} + O\left(\frac{1}{\epsilon^2}\right) \right) \Delta_j\\
		\implies {\Regret_T} &\le  \sum_{j > i^\star}\ \frac{(1+\epsilon)\ln(T)}{d(\pijs, \pijs + \xi_j)}\Delta_j + O\left(\frac{K-i^\star}{\epsilon^2}\right). \qedhere
	\end{align*}
\end{proof}

\indepRegretBound*
\begin{proof}
	Let $M_j(T)$ is the number of times arm $j$ preferred over the optimal arm in $T$ rounds. From \cref{lem:probLow}, for any $j<i^\star$, we have
	\begin{align*}
		\EE{M_j(T)} &= \sum_{t=1}^{T}\Prob{I_t = j, j < i^\star} \\
		&\le \frac{24}{\xi_j^2} + \sum_{s \ge 8/\xi_j} \Theta\left(\exp^{-{s\xi_j^2}/{2}} + \frac{\exp^{-{sd( \pijs - \xi_j,\pijs)}}}{(s+1)\xi_j^2} + \frac{1}{\exp^{{s\xi_j^2}/{4}} - 1} \right).
	\end{align*}
	It is east to show that $\frac{\exp^{-{sd( \pijs - \xi_j,\pijs)}}}{(s+1)\xi_j^2} \le \frac{1}{(s+1)\xi_j^2}$ and ${\exp^{{s\xi_j^2}/{4}} - 1}  \ge {s\xi_j^2}/4 $ (as $\exp^y \ge y+1$), 
	\begin{align*}
		\EE{M_j(T)} \le \frac{24}{\xi_j^2} + \sum_{s \ge 8/\xi_j} \Theta\left(\frac{1}{\xi_j^2} + \frac{1}{(s+1)\xi_j^2} + \frac{4}{s\xi_j^2} \right).
	\end{align*}
	By using $\sum_{s\ge 0}\exp^{-sa} \le 1/a$ and $\sum_{s=1}^T (1/s) = \log T$, 
	\begin{align*}
		\EE{M_j(T)} &\le \frac{24}{\xi_j^2} + \Theta\left(\frac{1}{\xi_j^2} + \frac{\ln T}{\xi_j^2} \right) 
		\implies \EE{M_j(T)} \le O\left(\frac{\ln T}{\xi_j^2}\right). \numberthis \label{eq:lowIndRegret}
	\end{align*}

	\noindent
	For any $j>i^\star$, using \cref{lem:probHighPart2} and \cref{lem:probHighPart1} with \cref{equ:selectToPrefer}, we have
	\begin{equation*}
		  \EE{M_j(T)} = \sum_{t=1}^T \Prob{I_t = j, j > i^\star} \le \sum_{t=1}^T \Prob{j \succ_t \ist, j > i^\star} \le  \frac{\ln T}{d(x_j, y_j)} + 1 + \frac{1}{d(x_j, \pijs)}.
	\end{equation*}
	By setting $x_j = \pijs + \frac{\xi_j}{3}$ and $y_j = \pijs + \frac{2\xi_j}{3}$, we have $d(x_j, y_j) \ge \frac{2\xi_j^2}{9}$ and $d(x_j, \pijs) \ge \frac{2\xi_j^2}{9}$ (using \cref{fact:pinsker}).
	\begin{align*}	
		\EE{M_j(T)} &\le \frac{9\ln T}{2\xi_j^2} + 1 + \frac{9}{2\xi_j^2}\\
		\implies \EE{M_j(T)} &\le O\left(\frac{\ln T}{\xi_j^2}\right). \numberthis	\label{eq:highIndRegret}
	\end{align*}

	\noindent
	The regret of \ref{alg:TS_USS} is given by
	\begin{align*}
		{\Regret_T} = \sum\limits_{j\neq i^\star} \EE{M_j(T)}\Delta_j = \sum\limits_{j<i^\star}\EE{M_j(T)}\Delta_j + \sum\limits_{j>i^\star}\EE{M_j(T)}\Delta_j 
	\end{align*}
	Recall $\Delta_j = C_j + \gamma_j - (C_{i^\star} + \gamma_{i^\star})$ and for any two arms $i$ and $j$, $0 \le \pij - (\gamma_j - \gamma_{i^\star} ) \le \beta$. By using Eq. \eqref{def_xi_l} for $j<i^\star$, we have $\Delta_j = \xi_j - (\pijs - (\gamma_{i^\star} - \gamma_j)) \implies \Delta_j \le \xi_j$, and using Eq. \eqref{def_xi_h} for $j>i^\star$, we have $\Delta_j = \xi_j + (\pijs - (\gamma_{i^\star} - \gamma_j)) \implies \Delta_j \le \xi_j + \beta$. Replacing $\Delta_j$, 
	\begin{equation*}
		\Rightarrow {\Regret_T} \leq \sum\limits_{j<i^\star}\EE{M_j(T)}\xi_j + \sum\limits_{j>i^\star}\EE{M_j(T)}(\xi_j + \beta). 
	\end{equation*} 
	Let $0<\xi^\prime<1$. Then  ${\Regret_T}$ can be written as:
	\begin{align*}
		{\Regret_T} & \leq \sum\limits_{\substack{\xi^\prime > \xi_j\\ j < i^\star}} \EE{M_j(T)}\xi_j + \sum\limits_{\substack{\xi^\prime < \xi_j\\ j < i^\star}} \EE{M_j(T)}\xi_j \\
		&\qquad + \sum\limits_{\substack{\xi^\prime > \xi_j\\ j > i^\star}} \EE{M_j(T)}(\xi_j + \beta) + \sum\limits_{\substack{\xi^\prime < \xi_j\\ j > i^\star}} \EE{M_j(T)}(\xi_j + \beta).
	\end{align*}
	Using $\sum\limits_{\xi^\prime > \xi_j} \EE{M_j(T)} \le T$ for any $j$ such that $\xi^\prime > \xi_j$, 
	\begin{align*}
		{\Regret_T} &  \le T\xi^\prime + \sum\limits_{\substack{\xi^\prime < \xi_j\\ j < i^\star}} \EE{M_j(T)}\xi_j + \sum\limits_{\substack{\xi^\prime < \xi_j\\ j > i^\star}} \EE{M_j(T)}(\xi_j + \beta).
	\end{align*}
	Substituting the value of ${\Regret_T}$ from \cref{eq:lowIndRegret} and \cref{eq:highIndRegret},
	\begin{align*}
		{\Regret_T} &  \le T\xi^\prime + \sum\limits_{\substack{\xi^\prime < \xi_j\\ j < i^\star}} O\left(\frac{\xi_j\ln T }{\xi_j^2}\right) + \sum\limits_{\substack{\xi^\prime < \xi_j\\ j > i^\star}} O\left(\frac{(\xi_j + \beta)\ln T }{\xi_j^2}\right)\\
		&\le T\xi^\prime + \sum\limits_{\substack{\xi^\prime < \xi_j\\ j < i^\star}} O\left(\frac{\ln T }{\xi_j}\right) + \sum\limits_{\substack{\xi^\prime < \xi_j\\ j > i^\star}} O\left(\frac{\ln T}{\xi_j} + \frac{\beta\ln T}{\xi_j^2}\right)\\
		&  \le T\xi^\prime + O\left(\frac{K\ln T }{\xi^\prime}\right) + O\left(\frac{K\ln T }{\xi^\prime} + \frac{\beta K\ln T }{{\xi^\prime}^2}\right)\\
		&  = T\xi^\prime + O\left(K\ln T\left(\frac{1}{\xi^\prime} + \frac{\beta}{{\xi^\prime}^2}\right)\right)
		\intertext{Let there exist a variable $\alpha$ such that $O\left(K\ln T\left(\frac{1}{\xi^\prime} + \frac{\beta}{{\xi^\prime}^2}\right)\right) \le \alpha K\ln T\left(\frac{1}{\xi^\prime} + \frac{\beta}{{\xi^\prime}^2}\right)$,}
		\implies {\Regret_T} & \le = T\xi^\prime + \alpha K\ln T\left(\frac{1}{\xi^\prime} + \frac{\beta}{{\xi^\prime}^2}\right). \numberthis \label{eq:indRegret}
	\end{align*}		
	Consider $\PWD$ class of problems. As $\xi^\prime < 1$ and $\beta \le 2$ (as arms in the cascade may not be ordered by their error-rates, it is possible that $\gamma_{i} < \gamma_{j}$), we have $\left(\frac{1}{\xi^\prime} + \frac{\beta}{{\xi^\prime}^2}\right) \le \frac{\beta + 1}{{\xi^\prime}^2}\le \frac{3}{{\xi^\prime}^2}$,
	\begin{align*}
		{\Regret_T} & \le = T\xi^\prime +\frac{3\alpha K\ln T}{{\xi^\prime}^2}.
		\intertext{Choose $\xi^\prime = \left( \frac{6\alpha K\ln T}{T}\right)^{1/3}$ which maximize above upper bound and we get,}
		{\Regret_T} &\le \left(6\alpha K\ln T\right)^{1/3}T^{2/3} + \frac{\left(6\alpha K\ln T\right)^{1/3}}{2}T^{2/3} \\
		\implies {\Regret_T} &\le 2\left(6\alpha K\ln T\right)^{1/3}T^{2/3} = O\left( \left(K\ln T\right)^{1/3}T^{2/3}\right)
	\end{align*}		
	It completes our proof for the case when any problem instance belongs to $\PWD$.
	
	~\\
	\noindent
	Now we consider any problem instance $\theta \in \PSD$. For any $\theta \in \PSD \Rightarrow \forall j \in [K],\; \pij = \gamma_{i} - \gamma_j \implies \beta = 0$ (Setting $\Prob{\Yi = Y, \Yj \ne Y} = 0$ for $j>i$ in Proposition 3 of \cite{AISTATS17_hanawal2017unsupervised}). We can rewrite \cref{eq:indRegret} as
	\begin{align*}
		{\Regret_T} &\le T\xi^\prime +\frac{\alpha K\ln T}{{\xi^\prime}}.
		\intertext{Choose $\xi^\prime = \left( \frac{\alpha K\ln T}{T}\right)^{1/2}$ which maximize above upper bound and we get,}
		\implies {\Regret_T} &\le 2\left(\alpha KT\ln T\right)^{1/2} = O\left( \sqrt{KT\ln T}\right) 
	\end{align*}
	This complete proof for second part of Theorem \ref{thm:indepRegretBound}.
\end{proof}

%% file: main.bbl
\begin{thebibliography}{25}
\providecommand{\natexlab}[1]{#1}
\providecommand{\url}[1]{\texttt{#1}}
\expandafter\ifx\csname urlstyle\endcsname\relax
  \providecommand{\doi}[1]{doi: #1}\else
  \providecommand{\doi}{doi: \begingroup \urlstyle{rm}\Url}\fi

\bibitem[Agrawal and Goyal(2012)]{COLT12_agrawal2012analysis}
Shipra Agrawal and Navin Goyal.
\newblock Analysis of thompson sampling for the multi-armed bandit problem.
\newblock In \emph{Conference on Learning Theory}, pages 39--1, 2012.

\bibitem[Agrawal and Goyal(2013)]{AISTATS13_agrawal2013further}
Shipra Agrawal and Navin Goyal.
\newblock Further optimal regret bounds for thompson sampling.
\newblock In \emph{Artificial intelligence and statistics}, pages 99--107,
  2013.

\bibitem[Alon et~al.(2013)Alon, Cesa-Bianchi, Gentile, and
  Mansour]{NIPS13_alon2013bandits}
Noga Alon, Nicolo Cesa-Bianchi, Claudio Gentile, and Yishay Mansour.
\newblock From bandits to experts: A tale of domination and independence.
\newblock In \emph{Advances in Neural Information Processing Systems}, pages
  1610--1618, 2013.

\bibitem[Alon et~al.(2015)Alon, Cesa-Bianchi, Dekel, and
  Koren]{COLT15_alon2015online}
Noga Alon, Nicolo Cesa-Bianchi, Ofer Dekel, and Tomer Koren.
\newblock Online learning with feedback graphs: Beyond bandits.
\newblock In \emph{Annual Conference on Learning Theory}, volume~40. Microtome
  Publishing, 2015.

\bibitem[Auer et~al.(2002)Auer, Cesa-Bianchi, and Fischer]{ML02_auer2002finite}
Peter Auer, Nicolo Cesa-Bianchi, and Paul Fischer.
\newblock Finite-time analysis of the multiarmed bandit problem.
\newblock \emph{Machine learning}, pages 235--256, 2002.

\bibitem[Bart{\'o}k and Szepesv{\'a}ri(2012)]{ALT12_bartok2012partial}
G{\'a}bor Bart{\'o}k and Csaba Szepesv{\'a}ri.
\newblock Partial monitoring with side information.
\newblock In \emph{International Conference on Algorithmic Learning Theory},
  pages 305--319. Springer, 2012.

\bibitem[Bart{\'o}k et~al.(2014)Bart{\'o}k, Foster, P{\'a}l, Rakhlin, and
  Szepesv{\'a}ri]{MOR14_bartok2014partial}
G{\'a}bor Bart{\'o}k, Dean~P Foster, D{\'a}vid P{\'a}l, Alexander Rakhlin, and
  Csaba Szepesv{\'a}ri.
\newblock Partial monitoring—classification, regret bounds, and algorithms.
\newblock \emph{Mathematics of Operations Research}, 39\penalty0 (4):\penalty0
  967--997, 2014.

\bibitem[Bonald and Combes(2017)]{NIPS17_bonald2017minimax}
Thomas Bonald and Richard Combes.
\newblock A minimax optimal algorithm for crowdsourcing.
\newblock In \emph{Advances in Neural Information Processing Systems}, pages
  4352--4360, 2017.

\bibitem[Bubeck et~al.(2012)Bubeck, Cesa-Bianchi,
  et~al.]{NOW12_bubeck2012regret}
S{\'e}bastien Bubeck, Nicolo Cesa-Bianchi, et~al.
\newblock Regret analysis of stochastic and nonstochastic multi-armed bandit
  problems.
\newblock \emph{Foundations and Trends{\textregistered} in Machine Learning},
  5\penalty0 (1):\penalty0 1--122, 2012.

\bibitem[Cesa-Bianchi et~al.(2006)Cesa-Bianchi, Lugosi, and
  Stoltz]{MOR06_cesa2006regret}
Nicolo Cesa-Bianchi, G{\'a}bor Lugosi, and Gilles Stoltz.
\newblock Regret minimization under partial monitoring.
\newblock \emph{Mathematics of Operations Research}, 31\penalty0 (3):\penalty0
  562--580, 2006.

\bibitem[Chapelle and Li(2011)]{NIPS11_chapelle2011empirical}
Olivier Chapelle and Lihong Li.
\newblock An empirical evaluation of thompson sampling.
\newblock In \emph{Advances in neural information processing systems}, pages
  2249--2257, 2011.

\bibitem[Chen et~al.(2012)Chen, Xu, Weinberger, Chapelle, and
  Kedem]{AISTATS12_chen2012classifier}
Minmin Chen, Zhixiang Xu, Kilian Weinberger, Olivier Chapelle, and Dor Kedem.
\newblock Classifier cascade for minimizing feature evaluation cost.
\newblock In \emph{Artificial Intelligence and Statistics}, pages 218--226,
  2012.

\bibitem[Detrano(1998)]{HEART98_robert1988va}
Robert Detrano.
\newblock {V.A. Medical Center, Long Beach and Cleveland Clinic Foundation:
  Robert Detrano, MD, Ph.D., Donor: David W. Aha}, 1998.
\newblock URL \url{https://archive.ics.uci.edu/ml/datasets/Heart+Disease}.

\bibitem[Garivier and Capp{\'e}(2011)]{COLT11_garivier2011kl}
Aur{\'e}lien Garivier and Olivier Capp{\'e}.
\newblock The kl-ucb algorithm for bounded stochastic bandits and beyond.
\newblock In \emph{Proceedings of the 24th annual Conference On Learning
  Theory}, pages 359--376, 2011.

\bibitem[Hanawal et~al.(2017)Hanawal, Szepesvari, and
  Saligrama]{AISTATS17_hanawal2017unsupervised}
Manjesh Hanawal, Csaba Szepesvari, and Venkatesh Saligrama.
\newblock Unsupervised sequential sensor acquisition.
\newblock In \emph{Artificial Intelligence and Statistics}, pages 803--811,
  2017.

\bibitem[Kaggle(2016)]{UCI16_pima2016kaggale}
UCI Machine~Learning{,} Kaggle.
\newblock {Pima Indians Diabetes Database}, 2016.
\newblock URL
  \url{https://www.kaggle.com/uciml/pima-indians-diabetes-database}.

\bibitem[Kaufmann et~al.(2012)Kaufmann, Korda, and
  Munos]{ALT12_kaufmann2012thompson}
Emilie Kaufmann, Nathaniel Korda, and R{\'e}mi Munos.
\newblock Thompson sampling: An asymptotically optimal finite-time analysis.
\newblock In \emph{International Conference on Algorithmic Learning Theory},
  pages 199--213. Springer, 2012.

\bibitem[Kleindessner and Awasthi(2018)]{ICLM18_kleindessner2018crowdsourcing}
Matth{\"a}us Kleindessner and Pranjal Awasthi.
\newblock Crowdsourcing with arbitrary adversaries.
\newblock In \emph{International Conference on Machine Learning}, pages
  2713--2722, 2018.

\bibitem[Lattimore and Szepesv{\'a}ri(2020)]{BOOK_lattimorebandit}
Tor Lattimore and Csaba Szepesv{\'a}ri.
\newblock Bandit algorithms, 2020.

\bibitem[Mannor and Shamir(2011)]{NIPS11_mannor2011bandits}
Shie Mannor and Ohad Shamir.
\newblock From bandits to experts: On the value of side-observations.
\newblock In \emph{Advances in Neural Information Processing Systems}, pages
  684--692, 2011.

\bibitem[Trapeznikov and Saligrama(2013)]{AISTATS13_trapeznikov2013supervised}
Kirill Trapeznikov and Venkatesh Saligrama.
\newblock Supervised sequential classification under budget constraints.
\newblock In \emph{Artificial Intelligence and Statistics}, pages 581--589,
  2013.

\bibitem[Verma et~al.(2019{\natexlab{a}})Verma, Hanawal, Rajkumar, and
  Sankaran]{NeurIPS19_verma2019censored}
Arun Verma, Manjesh Hanawal, Arun Rajkumar, and Raman Sankaran.
\newblock Censored semi-bandits: A framework for resource allocation with
  censored feedback.
\newblock In \emph{Advances in Neural Information Processing Systems}, pages
  14499--14509, 2019{\natexlab{a}}.

\bibitem[Verma et~al.(2019{\natexlab{b}})Verma, Hanawal, Szepesvari, and
  Saligrama]{AISTATS19_verma2019online}
Arun Verma, Manjesh Hanawal, Csaba Szepesvari, and Venkatesh Saligrama.
\newblock Online algorithm for unsupervised sensor selection.
\newblock In \emph{Artificial Intelligence and Statistics}, pages 3168--3176,
  2019{\natexlab{b}}.

\bibitem[Verma et~al.(2020)Verma, Hanawal, and
  Hemachandra]{COMSNETS20_verma2020unsupervised}
Arun Verma, Manjesh~K Hanawal, and Nandyala Hemachandra.
\newblock Unsupervised online feature selection for cost-sensitive medical
  diagnosis.
\newblock In \emph{2020 International Conference on COMmunication Systems \&
  NETworkS (COMSNETS)}, pages 1--6. IEEE, 2020.

\bibitem[Wu et~al.(2015)Wu, Gy{\"o}rgy, and
  Szepesv{\'a}ri]{NIPS15_wu2015online}
Yifan Wu, Andr{\'a}s Gy{\"o}rgy, and Csaba Szepesv{\'a}ri.
\newblock Online learning with gaussian payoffs and side observations.
\newblock In \emph{Advances in Neural Information Processing Systems}, pages
  1360--1368, 2015.

\end{thebibliography}
